\documentclass[12pt]{article}
\usepackage{mattsstyle}
\usepackage{todonotes}
\usepackage{subcaption}

\allowdisplaybreaks

\definecolor{darkred}{rgb}{0.6,0.1,0.1}
\definecolor{darkgreen}{rgb}{0.1,0.6,0.1}
\definecolor{darkblue}{rgb}{0.1,0.1,0.6}
\def\dd{\mathrm{d}}
\def\eps{\varepsilon}

\def\lp{\left(}
\def\rp{\right)}
\def\la{\left|}
\def\ra{\right|}
\def\ls{\left[}
\def\rs{\right]}
\def\bbP{\mathbb{P}}
\def\bbR{\mathbb{R}}
\def\rmc{\mathrm{c}}
\def\Ck#1{\mathrm{C}^{#1}}
\def\Lp#1{\mathrm{L}^{#1}}
\def\Hk#1{\mathrm{H}^{#1}}
\def\tr{\mathrm{tr}}
\def\oo#1{\stackrel{\circ}{#1}}

\newcommand{\joinR}{\hspace{-.1em}}
\newcommand{\RomanI}{\mbox{I}}
\newcommand{\RomanII}{\mbox{I\joinR I}}
\newcommand{\RomanIII}{\mbox{I\joinR I\joinR I}}
\newcommand{\RomanIV}{\mbox{I\joinR V}}
\newcommand{\RomanV}{\mbox{V}}
\newcommand{\RomanVI}{\mbox{V\joinR I}}

\begin{document}

\title{Large Data Limits of Laplace Learning for Gaussian Measure Data in Infinite Dimensions}
\author{Zhengang Zhong\thanks{Department of Statistics, University of Warwick, Coventry, CV4 7AL, UK}, \quad  Yury Korolev\thanks{Department of Mathematical Sciences, University of Bath, BA2 7AY, UK}, \quad  Matthew Thorpe$^*$}
\date{}
\maketitle

\begin{abstract}
Laplace learning is a semi-supervised method, a solution for finding missing labels from a partially labeled dataset utilizing the geometry given by the unlabeled data points.
The method minimizes a Dirichlet energy defined on a (discrete) graph constructed from the full dataset.
In finite dimensions the asymptotics in the large (unlabeled) data limit are well understood with convergence from the graph setting to a continuum Sobolev semi-norm weighted by the Lebesgue density of the data-generating measure.
The lack of the Lebesgue measure on infinite-dimensional spaces requires rethinking the analysis if the data aren't finite-dimensional.
In this paper we make a first step in this direction by analyzing the setting when the data are generated by a Gaussian measure on a Hilbert space and proving pointwise convergence of the graph Dirichlet energy. 
\end{abstract}



\section{Introduction}

Semi-supervised learning is a widely used technique across data science and engineering, leveraging both labeled and large amounts of unlabeled data to enhance model performance when labeled data is limited or expensive. Its goal is to design models that effectively assign labels to unlabeled data by exploiting the structure within the large set of unlabeled points. Laplace learning~\cite{zhou2005regularization,LapRef} (also known as Harmonic Extensions and Laplacian regularization) is one such technique that minimizes a Dirichlet energy defined on a graph. 

In many cases, such as for biomedical images \cite{bruggner2014automated, dudoit2002comparison}, shapes \cite{allaire2021shape}, and control policies \cite{troltzsch2010optimal}, data is infinite-dimensional. Motivated by these applications, we study the asymptotic behavior of Laplace learning and extend the pointwise convergence results of \cite{el2016asymptotic} to infinite-dimensional Gaussian measure data. 

In the semi-supervised learning setting, we typically assume that we are provided $N$ pairs of labeled data $\left\{\left(x_i, \ell_i\right): i=1, \ldots, N, x_i \in X, y_i \in \mathbb{R}\right\}$, where $x_i$ are the feature vectors and $y_i$ are the labels, and $n-N$ unlabeled feature vectors $\{x_i\}_{i=N+1}^n$. 
We assume that the feature vectors $x_i$ are drawn independently from a data-generating probability measure $\mu$. In this setting, a graph is constructed as a discrete proxy for the geometry and support of the measure $\mu$, using samples $x_i$. On the graph, the vertices correspond to the data points, and edges are assigned weights that quantify their proximity, typically defined through a kernel function at scale $\varepsilon>0$: 
$$ W_{i j}=\eta_{\varepsilon}\left(\left\|x_i-x_j\right\|\right), \quad \text { where } \eta_{\varepsilon}(r)= \eta(r / \varepsilon). $$

Building on this graph-based representation of data, we seek a labeling function $u$ that both fits the given labels and assigns labels to the unlabeled data in a smooth manner. To this end, we aim to solve a variational problem that promotes the smoothness of $u$ with respect the data structure encoded by the constructed graph: 
\begin{equation} \label{eq:Into:Energy}
\begin{aligned}
    \inf_u &  \sum_{i, j=1}^n W_{i j}\left|u\left(x_i\right)-u\left(x_j\right)\right|^p\\ 
    \text{s.t.} & \quad u\left(x_i\right)=y_i \text { for all } i=1, \ldots, N.
\end{aligned}
\end{equation}

While this formulation is defined for finitely many data points and a fixed, strictly positive kernel scale $\varepsilon$, it is of fundamental interest to analyze the behavior of the solution $u$ as the number of data points $n \rightarrow \infty$ and the kernel scale $\varepsilon \rightarrow 0$. This asymptotic limit, commonly referred to as the continuum, has been studied extensively from several perspectives, including pointwise convergence \cite{alberti1998non, coifman2006diffusion, el2016asymptotic, gine2006empirical, hein2006uniform, hein2005graphs, nadler2009semi, singer2006graph, ting2011analysis}, spectral convergence \cite{dunlop2020large, trillos2018variational, von2008consistency} and Gamma-convergence \cite{garcia2016continuum, slepcev2019analysis}.

In the finite-dimensional case, $X=\bbR^d,$ the energy in~\eqref{eq:Into:Energy} should be rescaled by $\frac{1}{n^2\eps^{d+p}}$, i.e. the energy of interest is
\[ \frac{1}{n^2\eps^{d+p}} \sum_{i,j=1}^n W_{ij} |u(x_i) - u(x_j)|^p. \]
Under suitable conditions, the results cited above imply convergence (in a pointwise, spectral and variational sense) to the continuum problem
\begin{align*}
    \inf _u & \int\|\nabla u(x)\|_2^p \rho^2(x) \, \dd x\\
    \text{s.t.} & \quad u\left(x_i\right)=y_i \text { for all } i=1, \ldots, N
\end{align*}
where $\mu$ has density $\rho$.
In the infinite-dimensional setting this continuum energy is no longer well defined since there is no Lebesgue measure and therefore the density $\rho = \frac{\dd \mu}{\dd x}$ does not exist. As a consequence, the finite-dimensional continuum formulation cannot be extended naively to infinite dimensions, and a reformulation of the problem is required.

In this work, $X$ will be an infinite-dimensional Hilbert space.
To avoid the problem of squared densities that come from the continuum limit of~\eqref{eq:Into:Energy} we consider the normalized objective function defined by 
$$
\mathcal{E}_{\eps,n}(u)= \frac{1}{n^2 \varepsilon^2} \sum_{i=1}^{n}\frac{\sum_{j=1}^{n} W_{ij}|u(x_i)-u(x_j)|^2}{d_i},
$$
where $d_{i} = \frac{1}{n} \sum_{j=1}^{n} \eta_{\varepsilon}\left(\left\|x_i-x_j\right\|_{X^\alpha}\right)$, $\eta_{\varepsilon}(r)= \eta(r / \varepsilon)$ and $\|\cdot\|_{X^\alpha}$ is a re-weighted norm on $X$ (precise definitions in the following section; if $\alpha=0$ then $\|\cdot\|_{X^\alpha} = \|\cdot\|_{X}$, the norm on $X$). The main difference compared to the finite-dimensional objective function is that, instead of normalizing the kernel by $\varepsilon^d$ (which is not meaningful in our infinite-dimensional setting), we normalize the weighted differences among points using the average weights.

As noted above, a fundamental difficulty in extending finite-dimensional semi-supervised learning schemes to infinite-dimensional settings is the absence of a probability density function. This precludes a direct use of density-based methods commonly employed in finite dimensions. To overcome this difficulty, we restrict our attention to Gaussian measure data, which possess a rich analytical structure and admit tractable representations. In particular, Gaussian measures allow us to analyze the objective functional via its Karhunen–Lo\`eve expansion, representing the measure as a product of infinitely many one-dimensional Gaussian components.

This infinite-dimensional perspective, frequently adopted in Gaussian measure settings, arises in a wide range of applications where the random variable of interest is a function rather than a finite-dimensional parameter. These include Bayesian inverse problems~\cite{stuart2010inverse} and non-parametric statistics~\cite{gine2021mathematical}. This perspective has led to a growing interest in learning problems formulated directly on infinite-dimensional spaces. Within this framework, diffusion-based generative models in infinite dimensions have been studied in~\cite{pidstrigach2024infinite}, where existence and uniqueness of the reverse-time stochastic differential equation are established, while~\cite{bunker2025autoencoders} develops functional variational autoencoders that can be trained and evaluated across multiple discretization levels.

Closely related to these developments, operator learning provides a data-driven framework for approximating mappings between function spaces. Rather than learning finite-dimensional parameter-to-solution maps at a fixed resolution, operator learning aims to learn discretization-invariant approximations of solution operators arising in partial differential equations and inverse problems. Neural operator architectures, including DeepONets \cite{lanthaler2022error} and Fourier Neural Operators \cite{li2020fourier}, are specifically designed for this purpose. Recent work has established a unified algorithmic and analytical foundation for operator learning, including approximation, stability, and convergence results formulated directly in infinite-dimensional settings \cite{kovachki2024operator}.

Our objective of this work is to characterize the asymptotic behavior of $\mathcal{E}_{\eps,n}(u)$ as the number of unlabeled data increases to infinity and the connection radius goes to zero. Specifically, for  fixed $u:X\to \bbR$, we study the pointwise limit
$\lim _{\varepsilon \rightarrow 0} \lim _{n \rightarrow \infty} \mathcal{E}_{\eps,n}(u)$ . The analysis proceeds in two consecutive steps: first, the discrete-to-nonlocal continuum limit as $n \rightarrow \infty$, and then the nonlocal-to-local continuum as $\varepsilon \rightarrow \infty$. While the discrete and continuum problems are well understood in the finite-dimensional setting, to the best of our knowledge there has been no work on semi-supervised learning with data living in infinite dimensions.

The paper is organized as follows. In Section \ref{sec:Setting}  we give a precise description of the problem with the assumptions and state the main results. Section \ref{sec:Back} contains some background on Gaussian measures over Hilbert spaces. Section \ref{sec:proofs} contains the proofs of the main results given in
Section \ref{sec:Setting}. In Section \ref{sec:num_exp}, we illustrate the flexibility and  performance of our approach through numerical experiments, comparing it to the usual finite-dimensional scheme, with a particular emphasis on the effect of norm selection. We conclude with some remarks in Section~\ref{sec:Conclusions}.

\section{Setting and Main Results} \label{sec:Setting}

We start with an assumption on the feature space. 

\begin{assumptions} \label{ass:Setting:X} 
We denote by $X$ a separable Hilbert space equipped with the inner product $\langle \cdot, \cdot\rangle_X$ and norm $\|\cdot\|_X$.
Let $X^*$ be its topological dual
\end{assumptions}

\begin{remark}
Assuming $X$ is a Hilbert space allows us to understand the Cameron-Martin space defined on the Gaussian measure in terms of the Hilbert space \cite[Page 116]{bogachev1998gaussian}, which will be explained in Section~\ref{subsec:Back:Gauss}. 
Additionally, the Hilbert space assumption implies that the space is reflexive, which allows us to understand the mean as an element of $X$, instead of the general bi-dual space.
\end{remark}

We consider a Gaussian probability measure $\mu$ on $X$. 
Here, a Gaussian measure on a Hilbert space $X$ is understood as a Borel measure such that $\mu \circ f^{-1}$ is a Gaussian on $\mathbb{R}^{1}$ for any $f \in X^{*}$ (see also Definition~\ref{def:Back:Gauss:Gauss}).
We assume that a sequence of unlabeled data $\{x_i\}_{i=1}^\infty$ are generated by i.i.d. by the measure $\mu$. 
The empirical measure induced by data points is given by $\mu_n := \frac{1}{n}\sum_{i=1}^{n}\delta_{x_i}$. 

\begin{assumptions} \label{ass:Setting:mu} 
The feature vectors $\{x_i\}_{i=1}^\infty$ are i.i.d. sample of a Gaussian measure $\mu$ with mean $m_\mu\in X$ and covariance operator $C_\mu:X\to X$. 
There exists an orthonormal eigenbasis $\{e_k\}_{k=1}^\infty$ of $X$, i.e. $\{e_k\}_{k=1}^\infty$ forms an orthonormal basis of $X$ and there exists $\{\lambda_i\}_{i=1}^\infty$ such that $C_\mu e_i=\lambda_i e_i$.
We assume that $\sum_{i=1}^\infty \lambda_i < \infty$ (i.e. $C_\mu$ is trace-class).
\end{assumptions}

Since the covariance operator is (by definition) positive semi-definite then all eigenvalues are non-negative, i.e. $\lambda_i\geq 0$.
Without loss of generality we will assume that $\lambda_1\geq \lambda_2\geq ...$.
With only a marginal loss of generality (see the discussion in Section~\ref{subsec:Back:Gauss} we will also assume that $\lambda_i>0$ for all $i\in\bbN$.

We use the spectrum of $C_\mu$ to define the fractional norm $\|\cdot\|_{X^\gamma}$ and fractional inner product $\langle\cdot,\cdot,\rangle_{X^\gamma}$: for any $\gamma\in\bbR$
\begin{equation}\label{eq:frac-norm}
    \langle x, y\rangle_{X^\gamma}: = \sum_{i=1}^\infty \lambda^\gamma \langle x,e_i\rangle_X \langle y,e_i\rangle_X \qquad \qquad \text{and} \qquad \qquad \|x\|_{X^\gamma} := \sqrt{\langle x,x\rangle_{X^\gamma}}.
\end{equation}
Of course $\langle\cdot,\cdot\rangle_{X^0} = \langle \cdot,\cdot\rangle_X$ and the associated spaces and norms are also equal.
We also have $\|x\|_{X^\gamma} = \sqrt{\sum_{i=1}^\infty \lambda_i^\gamma \langle x,e_i\rangle_X^2}$, which implies for $\gamma>0$,
\[ \|x\|_{X^\gamma} \leq \lambda_1^{\frac{\gamma}{2}} \|x\|_X \qquad \qquad \text{and} \qquad \qquad \|x\|_X \leq \lambda_1^{\frac{\gamma}{2}} \|x\|_{X^{-\gamma}}. \]
The fractional space $X^\gamma$ is defined as the completion (under the $\|\cdot\|_{X^\gamma}$ norm) of $X$, i.e. $X^\gamma = \overline{\lb x\in X\, :\, \|x\|_{X^\gamma}<\infty\rb}^{\|\cdot\|_{X^\gamma}}$.
For $\gamma>0$, $X\subset X^\gamma$ and for $\gamma<0$, $X^\gamma\subset X$.

We define $G_n := (\Omega_n, W_n)$ to be the graph with vertices $\Omega_n = \{ x_i:i=1,\dots,n\}$ and edge weights $W_n = \{W_{ij} \}_{i,j=1}^{n}$. 
For notational simplicity we will say there is no edge between $x_i$ and $x_j$ if $W_{ij} = 0$.
We consider the following edge weights 
\begin{equation} \label{eq:Setting:weight_function}
W_{ij}=\eta_{\varepsilon}\left(\left\|x_i-x_j\right\|_{X^\alpha}\right)
\end{equation}
where $\eta_{\varepsilon}:= \eta(\cdot / \varepsilon)$ with $\varepsilon$ characterizing the connection radius.
The form of the edge weights allow us to choose the norm (or more precisely the weighting $\alpha$).
At this stage $\alpha\in \bbR$, and in particular can be positive or negative.
We use the explicit form for $\eta$ given in the next assumption.

\begin{assumptions} \label{ass:Setting:eta} 
The kernel function $\eta:[0, \infty) \rightarrow \mathbb{R}$ is defined by $\eta(t) = e^{-t^2}$.
\end{assumptions}

The general regression problem of interest for semi-supervised learning problems we consider is to find an estimator $u: X \rightarrow \mathbb{R}$ which, under some constraints, minimizes the objective function 
$$
\mathcal{E}_{\eps,n}(u)= \frac{1}{n^2 \varepsilon^2} \sum_{i=1}^{n}\frac{\sum_{j=1}^{n} W_{ij}|u(x_i)-u(x_j)|^2}{d_i},
$$
where $d_{i} = \frac{1}{n} \sum_{j=1}^{n} W_{ij}$. 
We are interested in investigating the asymptotic behavior when the number of data points, $n$, goes to infinity and the connection radius, $\eps$, goes to zero.
Namely
$\lim _{\varepsilon \rightarrow 0} \lim _{n \rightarrow \infty} \mathcal{E}_{\eps,n}(u)$ for fixed $u:X\to \bbR$. 

Our results will show the interplay between the choice of norm in the weights and the regularity of $u$.
We will consider functions that are twice Fr\'echet differentiable on $X^\beta$. 
For a function $u:X^\beta\to \bbR$ to be Fr\'echet differentiable at $x\in X^\beta$ means there exists $D_{X^\beta}u(x)\in X^\beta$ and $R_{1,x}:X^\beta\to \bbR$ such that
\[ u(x+y) - u(x) - \langle D_{X^\beta}u(x),y\rangle_{X^\beta} = R_{1,x}(y) \]
and $\lim_{\|y\|_{X^\beta}\to 0} \frac{|R_{1,x}(y)|}{\|y\|_{X^\beta}} = 0$.
For a function $u:X^\beta\to \bbR$ to be twice Fr\'echet differentiable at $x\in X^\beta$ means $u$ is Fr\'echet differentiable and for every $y\in X^\beta$ there exists $D_{X^\beta}^2u(x;y)\in X^\beta$ and $R_{2,x,y}:X^\beta\to \bbR$ such that
\[ \langle D_{X^\beta} u(x+z),y\rangle_{X^\beta} - \langle D_{X^\beta} u(x),y\rangle_{X^\beta} - \langle D_{X^\beta}^2u(x;y),z\rangle_{X^\beta} = R_{2,x,y}(z) \]
and $\lim_{\|z\|_{X^\beta}\to 0} \frac{|R_{2,x,y}(z)|}{\|z\|_{X^\beta}} = 0$ where the convergence is uniform over bounded $y$.
If $u$ is 
\begin{enumerate}
\item twice Fr\'echet differentiable in $X^\beta$ for all $x\in X^\beta$;
\item $h\mapsto D_{X^\beta}^2u(x;h)$ is linear for every $x\in X^\beta$;
\item $x\mapsto D_{X^\beta}^2u(x;h)$ is continuous; and
\item $\sup_{x\in X^\beta} \sup_{h\in X^\beta} \frac{\|D_{X^\beta}^2 u(x;h)\|_{X^\beta}}{\|h\|_{X^\beta}}<+\infty$
\end{enumerate}
then we write $u\in\Ck{2}(X^\beta)$.

We will use Taylor's theorem for Fr\'echet differentiable functions (in fact one only needs the function to be twice Gateaux differentiable):
\[ u(x+h) - u(x) = \langle D_{X^\beta} u(x), h\rangle_{X^\beta} + \int_0^1 (1-t) \langle D_{X^\beta}^2 u(x+th;h),h\rangle_{X^\beta} \, \dd t. \]

We are now in a position to state our main result.

\begin{theorem}[Pointwise Consistency]
\label{thm:Setting:PointCons}
Suppose Assumptions~\ref{ass:Setting:X} - \ref{ass:Setting:eta} are satisfied.
Let $\alpha\in (-1,+\infty)$ and $\beta\in\bbR$.
Assume $u\in\Ck{2}(X^\beta)$, $\sum_{i=1}^\infty \lambda_i^{\beta+\frac12-\frac{\alpha}{2}}<+\infty$, $\sum_{i=1}^\infty\lambda_i^{\alpha+1}<+\infty$, $\sum_{i=1}^\infty\lambda_i^{2(\beta+1)}<+\infty$ and $m_\mu\in X^\alpha\cap X^\beta$.
Then we have 

\[ \lim_{\eps \to 0} \lim_{n \to \infty} \frac{1}{n^2 \eps^2} \sum_{i=1}^{n}\frac{\sum_{j=1}^{n} W_{ij}|u(x_i)-u(x_j)|^2}{d_i} = \frac12 \int_X  \left \| D_{X^\beta} u(x)\right\|_{X^{2\beta-\alpha}} ^2 \, \dd \mu(x) \]
almost surely.
\end{theorem}

The proof is an application of Corollary~\ref{cor:proofs:LDC:LDC} and Theorem~\ref{thm:proofs:nonlocal:nonlocal}.

\begin{remark}
If $y\in X^{\gamma_1}$ then, since $\|y\|_{X^{\gamma_1+\gamma_2}} \leq \lambda_1^{\frac{\gamma_2}{2}} \|y\|_{X^{\gamma_1}}$ for all $\gamma_2\geq 0$, $y\in X^{\gamma_2}$.
Hence, $D_{X^\beta}u(x)\in X^\beta$ (by assumption) implies that $\|D_{X^\beta}u(x)\|_{X^{2\beta-\alpha}}<+\infty$ if $\beta>\alpha$ (which is implied by $\beta> \frac{1+3\alpha}{2}$ and $\alpha>-1$). 
In particular, the right hand side of Theorem~\ref{thm:Setting:PointCons} is well-defined.
\end{remark}

\begin{remark}
We note that the condition $\sum_{i=1}^\infty \lambda_i^{\beta+\frac12-\frac{\alpha}{2}}$ hold for $\alpha=0$ and $\beta = 0$ if $\sum_{i=1}^\infty \sqrt{\lambda_i}<+\infty$. 
That is, we can use the original (non-weighted) norm on $X$ if the eigenvalues of the covariance decay sufficiently fast.
\end{remark}

\section{Background Material} \label{sec:Back}

In an effort to make this paper more self-contained, we briefly recall infinite-dimensional Gaussian distributions in separable Hilbert spaces and derive a few technical results that will be useful in Section~\ref{sec:proofs}.

\subsection{Gaussian Distributions}
\label{subsec:Back:Gauss}

Let $X$ be a separable Hilbert space and denote $f_\ell(x):= \langle x, \ell\rangle_X$ for $\ell \in X$.
We denote the corresponding $\sigma$-field by $\mathcal{E}(X)$.

\begin{definition}[Gaussian measures on separable Hilbert space] \label{def:Back:Gauss:Gauss}
Let $X$ be a separable Hilbert space. A probability measure $\mu$ on $X$ (or more precisely on $\mathcal{E}(X)$) is called Gaussian if, for any $\ell \in X$, the induced measure $\mu \circ f_\ell^{-1}$ on $\mathbb{R}^1$ is Gaussian. The measure $\mu$ is called centered if all the measures $\mu \circ f_\ell^{-1}$, $\ell \in X$, are centered.
\end{definition}

In the following we define the mean and covariance operator for the Gaussian measure $\mu$.

\begin{definition}
Let $X$ be a separable Hilbert space and let $\mu$ be a Gaussian measure on $X$. 
Then $m_\mu \in X$ defined by 
$$ m_\mu := \int_{X}  x \, \dd \mu(x) $$
is called the mean of $\mu$.
For a centered Gaussian measure $\mu$, we define $C_{\mu}: X \rightarrow  X $ by
$$ C_{\mu}:= \int_{X} \langle \cdot,x \rangle_X  x   \, \dd \mu(x), $$ 
to be the covariance operator of $\mu$. 
\end{definition}

Following the standard notation in Euclidean spaces we will write $\cN(m_\mu,C_\mu)$ for the Gaussian measure with mean $m_\mu$ and covariance $C_\mu$.

One can check that $C_\mu$ is a bounded linear operator and moreover $\tr\,C_\mu = \int_X \|x\|_X^2 \, \dd \mu(x) <+\infty$, so $C_\mu$ is trace-class and therefore compact (see also~\cite[Proposition 4.17]{hairer2009introduction}).
In particular, there exists an orthonormal basis of eigenfunctions $\{e_k\}_{k=1}^\infty$ with eigenvalues $\{\lambda_i\}_{i=1}^\infty$ such that $C_\mu e_i = \lambda_i e_i$.
Since $\lambda_i = \langle C_\mu e_i,e_i\rangle_X = \int_X \langle e_i,x\rangle_X^2 \, \dd \mu(x) \geq 0$ we have that the eigenvalues are all non-negative.
If we let $T:X\to \bbR$ be defined by $T(x) = \langle x,e\rangle_X$ (the projection onto the eigenfunction $e_i$) then we have $T_{\#}\mu = \cN(T(m_\mu),\lambda_i)$, the Gaussian distribution on $\bbR$ with mean $T(m_\mu)\in \bbR$ and variance $\lambda_i$.
If any $\lambda_i = 0$ there is no ``randomness in the direction $e_i$''.
In particular, if only finitely many $\lambda_i$ are non-zero then the problem is finite-dimensional.
Since we wish to concentrate on the infinite-dimensional case we will assume that all $\lambda_i>0$.
Hence, $C_\mu$ is invertible.

We are now ready to introduce the Cameron-Martin space, which roughly speaking defines directions where the shifted Gaussian measures remain absolutely continuous with respect to the original one.

\begin{definition}[Cameron-Martin space]
Let $X$ be a separable Hilbert space and $\mu$ be a Gaussian measure with covariance operator $C_\mu$. 
The Cameron-Martin space $\mathcal{H}_\mu$ of $\mu$ is the completion of the linear subspace $\stackrel{\circ}{\mathcal{H}}_\mu \subset X$ defined by
$$
\stackrel{\circ}{\mathcal{H}}_\mu=\left\{h \in X: \exists h^* \in X \text { with } \langle C_\mu h^*, \ell\rangle =\langle \ell, h\rangle \forall \ell \in X\right\},
$$
under the norm $\|h\|_\mu^2=\langle h, h\rangle_\mu = \langle h, C_\mu^{-1} h\rangle_X$.
The corresponding inner product is defined by $\langle h,g\rangle_\mu = \langle h, C_\mu^{-1} g\rangle_X = \langle C_\mu^{-\frac12} h, C_\mu^{-\frac12} g\rangle_X$.
\end{definition}

\begin{remark} 
If $\{e_k\}_{k=1}^n$ are an orthonormal basis of $X$ then it is straightforward to check that $\oo{e}_i = \frac{1}{\sqrt{\lambda_i}} e_i$ is an orthonormal basis of $\oo{\cH}_\mu$.
Indeed:
\[ \delta_{ij} = \langle e_i,e_j \rangle_X = \frac{1}{\sqrt{\lambda_i\lambda_j}} \langle C_\mu^{\frac12} e_i, C_\mu^{\frac12} e_j\rangle_X = \frac{1}{\sqrt{\lambda_i\lambda_j}} \langle e_i, e_j\rangle_\mu = \langle \oo{e}_i, \oo{e}_j \rangle_\mu. \]
\end{remark}

\begin{remark}
We can associate the Cameron-Martin space with the fractional space $X^{-1}$.
If we consider the inner products, then
\begin{align*}
\langle x,y\rangle_\mu & = \langle x,C_\mu^{-1}y\rangle_X \\
 & = \sum_{i=1}^\infty \langle x,e_i\rangle_X \langle C^{-1}_\mu y,e_i\rangle_X \\
 & = \sum_{i=1}^\infty \langle x,e_i\rangle_X \langle y,C_\mu^{-1} e_i\rangle_X \\
 & = \sum_{i=1}^\infty \lambda_i^{-1} \langle x,e_i\rangle_X \langle y, e_i\rangle_X \\
 & = \langle x,y\rangle_{X^{-1}}.
\end{align*}
So $X^{-1} = \oo{\cH}_\mu$.
\end{remark}

The next lemma gives a practical way to compute integrals with respect to Gaussian measures.

\begin{lemma}
\label{lem:Back:Gauss:KL_expansion}
Suppose Assumption~\ref{ass:Setting:X} holds.
Let $\gamma = \cN(m_\mu,C_\gamma)$ be a Gaussian measure on $X$ and let $\{(\lambda_i,e_i)\}_{i=1}^\infty$ be the eigenpairs of $C_\mu$. 
Let $(\Omega,\cF,P)$ be a probability space and $\xi:\Omega\to \bbR^\infty$ be the random variable such that $\xi(\omega)_i$ (the projection onto the $i$th coordinate) are iid standard Gaussian distributed on $\bbR$.
Then,  
\[ \int_X f(x) \,\dd\mu(x)=\int_{\Omega} f\left(\sum_{i=1}^{\infty} \sqrt{\lambda_i}\xi(\omega)_i e_i + m_\mu\right) \, \dd P(\omega) \]
for any measurable $f:X\to \bbR$.
\end{lemma}

We can also formally write
\[ \int_X f(x) \,\dd\mu(x) = \int_{\mathbb{R}^{\infty}} f\left(\sum_{i=1}^{\infty} \sqrt{\lambda_i}\xi_i e_i + m_\mu\right) \ls \prod_{i=1}^{\infty} \rho(\xi_i) \rs \, \dd \xi_1 \, \dd \xi_2 \dots \]
where $\rho$ is the standard Gaussian density on $\bbR$.

\begin{proof}[Proof of Lemma~\ref{lem:Back:Gauss:KL_expansion}.]
First we assume that $m_\mu=0$.
Let $\oo{e}_i=\frac{1}{\sqrt{\lambda_i}} e_i$, then $\{\oo{e}_i\}_{i=1}^\infty$ is an orthonormal basis for $\oo{\cH}_\mu$.
By~\cite[Theorem 3.5.1]{bogachev1998gaussian} (more specifically~\cite[Eq. (3.5.4)]{bogachev1998gaussian}) we have
\[ \int_X f(x) \,\dd\mu(x)=\int_{\Omega} f\left(\sum_{i=1}^{\infty} \xi(\omega)_i \oo{e}_i\right) \, \dd P(\omega) = \int_{\Omega} f\left(\sum_{i=1}^{\infty} \sqrt{\lambda_i}\xi(\omega)_i e_i\right) \, \dd P(\omega). \]

To deal with the general case ($m_\mu\neq 0$) we define $T(x) = x-m_\mu$ so that $T_{\#}\mu = \cN(0,C_\mu)$ and we apply the first part to $T_{\#}\mu$.
In particular,
\begin{align*}
\int_X f(x) \, \dd \mu(x) & = \int_X f(T^{-1}(x)) \, \dd T_{\#}\mu(x) \\
 & = \int_\Omega f(x+m_\mu) \, \dd T_{\#}\mu(x) \\
 & = \int_\Omega f\l \sum_{i=1}^\infty \sqrt{\lambda_i} \xi(\omega)_i e_i + m_\mu\r \, \dd P(\omega). \qedhere
\end{align*}
\end{proof}

We also check the regularity of samples in $\mu$.

\begin{proposition}
\label{prop:Back:Gauss:mu_support}
Under Assumptions~\ref{ass:Setting:X} and~\ref{ass:Setting:mu} we have $\supp(\mu) \subset X^\gamma$ for any $\gamma$ such that $\sum_{i=1}^\infty \lambda_i^{2(\gamma+1)}<+\infty$ and $m_\mu\in X^\gamma$. 
\end{proposition}

\begin{proof}
We can (by the Karhunen-Lo\'eve Theorem) write $x\sim \mu$ by $x=\sum_{i=1}^\infty \sqrt{\lambda_i} \xi_i e_i + m_\mu$ where $\xi_i\iid \cN(0,1)$ are i.i.d. $\bbR$-valued random variables (see also Lemma~\ref{lem:Back:Gauss:KL_expansion}).
Then,
\[ \|x\|_{X^\gamma}^2 = \sum_{i=1}^\infty \lambda_i^\gamma \langle x,e_i\rangle_X^2 = \sum_{i=1}^\infty \lambda_i^\gamma \l \sqrt{\lambda_i} \xi_i + m_i\r^2 \leq 2\sum_{i=1}^\infty \lambda_i^{\gamma+1}\xi_i^2 + 2 \underbrace{\sum_{i=1}^\infty \lambda_i^\gamma m_i^2}_{=\|m_\mu\|_{X^\gamma}^2} \]
where $m_i = \langle m_\mu,e_i\rangle_X$.
By~\cite[Theorem 2.5.6]{durrett19} $\sum_{i=1}^\infty \lambda_i^{\gamma+1} \xi_i^2 < +\infty$ almost surely if $\sum_{i=1}^\infty \lambda_i^{2(\gamma+1)}<+\infty$.
\end{proof}

\subsection{Technical Preliminary Results} \label{subsec:Back:Tech}

In this section we include some background preliminary results that will be useful in the sequel.
We start with a ``conjugacy'' relationship between a Gaussian measure $\mu$ with a Gaussian kernel $\eta$.

\begin{lemma} \label{lem:Back:Tech:Conj}
Let Assumptions~\ref{ass:Setting:X} and~\ref{ass:Setting:eta} hold.
Assume $\mu=\cN(m_\mu,C_\mu)$ and $\{(\lambda_i,e_i)\}_{i=1}^\infty$ are the eigenpairs of $C_\mu$.
For all $\alpha\in\bbR$, $\eps >0$, $x\in X^\alpha$ and measurable $f:X\to \bbR$, if $m_\mu\in X^\alpha$ and $\sum_{i=1}^\infty \lambda_i^{\alpha+1}<+\infty$ then we have 
\[ \frac{\int_X \eta\left(\frac{\|y - x\|_{X^\alpha}}{\varepsilon}\right) f\lp y-x\rp \, \dd\mu(y)}{ \int_X \eta\left(\frac{\|y -x\|_{X^\alpha}}{\varepsilon}\right)\,\dd\mu(y) }= \int_X f(y-x)\,\dd\mu_{x,\varepsilon}(y), \]
where 
$\mu_{x,\eps} = \cN(m_{x,\eps},C_{x,\eps})$
with $m_i = \langle m_\mu,e_i\rangle_X$, $x_i = \langle x,e_i\rangle_X$, and
\begin{align}
m_{x,\eps} & = \sum_{i=1}^\infty \frac{\eps^2 m_i + 2\lambda_i^{\alpha+1} x_i}{\eps^2+2\lambda_i^{\alpha+1}}e_i , \label{eq:Back:Tech:mxeps} \\
C_{x,\eps} & = \l C_\mu^{-1} + 2\eps^{-2} C_\mu^\alpha\r^{-1}. \label{eq:Back:Tech:Cxeps}
\end{align}
\end{lemma}

It's straightforward to check that $C_{x,\eps}$ has the same eigenfunctions $e_i$ as $C_\mu$ and eigenvalues $\frac{\eps^2\lambda_i}{\eps^2+2\lambda_i^{\alpha+1}}$, i.e. $C_{x,\eps} e_i = \frac{\eps^2\lambda_i}{\eps^2+2\lambda_i^{\alpha+1}} e_i$.

\begin{proof}[Proof of Lemma~\ref{lem:Back:Tech:Conj}.]
Let $f:X\to \bbR$ be a indicator function over a cylindrical set.
That is, there exists $n\in\bbN$, a continuous linear map $\varphi:X\to\bbR^d$ and a measurable function set $A\subset\bbR^d$ such that $f(x)=\one_{\{\varphi(x)\in A\}}$.
If we can prove the result for $f$ of this form, then by linearity we have that the result holds for all step functions over cylindrical sets.
Since the $\sigma$-algebra generated by cylindrical sets coincides with the $\sigma$-algebra generated by open sets (e.g.~\cite[Lemma 4.4]{eldredge2016analysis}) 
then we can approximate all measurable functions by a sequence of step functions over cylindrical sets.
Hence the result holds for all measurable functions.
Without loss of generality (by choosing the right basis in $\bbR^d)$ we can assume that $\varphi(e_i) = (0,\dots,0,\underset{i\text{th position}}{1},0,\dots,0)$ for $i\in\{1,\dots, d\}$ and $\varphi(e_i) = 0$ otherwise. 
For ease of notation we let $\tilde{f}(z)=\one_{\{z\in A\}}$.

Let $N_{\varepsilon}(x):= \int_X \eta\left(\frac{\|y - x\|_{X^\alpha}}{\varepsilon}\right) f(y-x) \, \dd\mu(y)$ and $D_{\varepsilon}(x):= \int_X \eta\left(\frac{\|y -x\|_{X^\alpha}}{\varepsilon}\right)\,\dd\mu(y)$.
We apply Lemma~\ref{lem:Back:Gauss:KL_expansion} 
to $N_{\varepsilon}(x)$: 
\begin{align*}
N_\eps(x) & = \int_X e^{-\frac{1}{\eps^2} \sum_{i=1}^\infty \lambda_i^\alpha (y_i-x_i)^2} f(y-x) \, \dd \mu(x) \\
 & = \int_\Omega e^{ -\frac{1}{\eps^2} \sum_{i=1}^\infty \lambda_i^\alpha \l \sqrt{\lambda_i} \xi(\omega)_i + m_i - x_i\r^2} \tilde{f}\l \sum_{i=1}^\infty \sqrt{\lambda_i}\xi(\omega)_i \varphi(e_i) + \varphi(m_\mu) - \varphi(x)\r \, \dd P(\omega) \\
 & = \int_\Omega e^{-\frac{1}{\eps^2} \sum_{i=1}^\infty \lambda_i^\alpha \l \sqrt{\lambda_i}\xi(\omega)_i + m_i-x_i\r^2} \tilde{f}\l \sum_{i=1}^d \l\sqrt{\lambda_i}\xi(\omega)_i + m_i - x_i\r\varphi(e_i) \r \, \dd P(\omega) \\
 & = \frac{1}{(2\pi)^{\frac{d}{2}}} \int_{\bbR^d} e^{-\frac{1}{\eps^2} \sum_{i=1}^d \lambda_i^\alpha\l\sqrt{\lambda_i} \xi_i + m_i-x_i\r^2 - \frac12 \sum_{i=1}^d \xi_i^2} \tilde{f}\l \sum_{i=1}^d \l\sqrt{\lambda_i}\xi_i + m_i - x_i\r\varphi(e_i) \r \, \dd \xi_{1:d} \\
 & \qquad \,\, \times \ls \prod_{i=d+1}^\infty \frac{1}{\sqrt{2\pi}} \int_\bbR e^{-\frac{\lambda_i^\alpha}{\eps^2} \l\sqrt{\lambda_i} \xi_i + m_i-x_i\r^2 - \frac{\xi_i^2}{2}} \, \dd \xi_i \rs \\
 & = \frac{\exp\l-\sum_{i=1}^d \frac{\lambda_i^\alpha(m_i-x_i)^2}{2\lambda_i^{\alpha+1}+\eps^2}\r}{(2\pi)^{\frac{d}{2}}} \int_{\bbR^d} \exp\l-\frac{1}{2} \sum_{i=1}^d \frac{2\lambda_i^{\alpha+1}+\eps^2}{\eps^2} \l \xi_i + \frac{2\lambda_i^{\alpha+\frac12}(m_i-x_i)}{2\lambda_i^{\alpha+1}+\eps^2} \r^2\r \\
 & \qquad \,\, \times \tilde{f}\l \sum_{i=1}^d \l\sqrt{\lambda_i}\xi_i + m_i - x_i\r\varphi(e_i) \r \, \dd \xi_{1:d} \\
 & \qquad \,\, \times \ls \prod_{i=d+1}^\infty \frac{\exp\l-\frac{\lambda_i^\alpha(m_i-x_i)^2}{2\lambda_i^{\alpha+1}+\eps^2}\r}{\sqrt{2\pi}} \int_\bbR \exp\l-\frac{2\lambda_i^{\alpha+1}+\eps^2}{2\eps^2} \l \xi_i + \frac{2\lambda_i^{\alpha+\frac12}(m_i-x_i)}{2\lambda_i^{\alpha+1}+\eps^2}\r^2\r \, \dd \xi_i \rs \\
 & = \frac{1}{(2\pi)^{\frac{d}{2}}}  \int_{\bbR^d} \exp\l-\frac{1}{2} \sum_{i=1}^d \frac{2\lambda_i^{\alpha+1}+\eps^2}{\eps^2} \l \xi_i + \frac{2\lambda_i^{\alpha+\frac12}(m_i-x_i)}{2\lambda_i^{\alpha+1}+\eps^2} \r^2\r \\
 & \qquad \,\, \times \tilde{f}\l \sum_{i=1}^d \l\sqrt{\lambda_i}\xi_i + m_i - x_i\r\varphi(e_i) \r \, \dd \xi_{1:d} \\
 & \qquad \,\, \times \exp\l-\sum_{i=1}^\infty \frac{\lambda_i^\alpha(m_i-x_i)^2}{2\lambda_i^{\alpha+1}+\eps^2}\r \ls \prod_{i=d+1}^\infty \frac{\eps}{\sqrt{2\lambda_i^{\alpha+1}+\eps^2}}\rs \\
 & = \frac{1}{(2\pi)^{\frac{d}{2}}} \ls \prod_{i=1}^d \frac{\eps}{\sqrt{2\lambda_i^{\alpha+1}+\eps^2}} \rs \int_{\bbR^d} \exp\l-\frac{1}{2} \sum_{i=1}^d \xi_i^2\r \\
 & \qquad \,\, \times \tilde{f}\l \sum_{i=1}^d \l\sqrt{\frac{\eps^2\lambda_i}{2\lambda_i^{\alpha+1}+\eps^2}}\xi_i + \frac{\eps^2(m_i-x_i)}{2\lambda_i^{\alpha+1}+\eps^2}\r\varphi(e_i) \r \, \dd \xi_{1:d} \\
 & \qquad \,\, \times \exp\l-\sum_{i=1}^\infty \frac{\lambda_i^\alpha(m_i-x_i)^2}{2\lambda_i^{\alpha+1}+\eps^2}\r \ls \prod_{i=d+1}^\infty \frac{\eps}{\sqrt{2\lambda_i^{\alpha+1}+\eps^2}}\rs \\
 & = \exp\l-\sum_{i=1}^\infty \frac{\lambda_i^\alpha(m_i-x_i)^2}{2\lambda_i^{\alpha+1}+\eps^2}\r \ls \prod_{i=1}^\infty \frac{\eps}{\sqrt{2\lambda_i^{\alpha+1}+\eps^2}}\rs \\
 & \qquad \,\, \times \int_\Omega f\l \sum_{i=1}^\infty \l\sqrt{\frac{\eps^2\lambda_i}{2\lambda_i^{\alpha+1}+\eps^2}}\xi(\omega)_i + \frac{\eps^2(m_i-x_i)}{2\lambda_i^{\alpha+1}+\eps^2}\r e_i \r \, \dd P(\omega) \\
 & = \exp\l-\sum_{i=1}^\infty \frac{\lambda_i^\alpha(m_i-x_i)^2}{2\lambda_i^{\alpha+1}+\eps^2}\r \ls \prod_{i=1}^\infty \frac{\eps}{\sqrt{2\lambda_i^{\alpha+1}+\eps^2}}\rs \\
 & \qquad \,\, \times \int_\Omega f\l \sum_{i=1}^\infty \l\sqrt{\frac{\eps^2\lambda_i}{2\lambda_i^{\alpha+1}+\eps^2}}\xi(\omega)_i + \frac{\eps^2 m_i+2\lambda_i^{\alpha+1} x_i}{2\lambda_i^{\alpha+1}+\eps^2}\r e_i - x \r \, \dd P(\omega) \\
 & = \exp\l-\sum_{i=1}^\infty \frac{\lambda_i^\alpha(m_i-x_i)^2}{2\lambda_i^{\alpha+1}+\eps^2}\r \ls \prod_{i=1}^\infty \frac{\eps}{\sqrt{2\lambda_i^{\alpha+1}+\eps^2}}\rs \int_X f(y-x) \, \dd \mu_{x,\eps}(y)
\end{align*}
where $m_i = \langle m_\mu,e_i\rangle_X$ and $x_i = \langle x,e_i\rangle_X$. 

Now, choosing $f\equiv 1$ we can easily infer
\[ N_\eps(x) = e^{-\sum_{i=1}^\infty \frac{\lambda_i^\alpha(m_i-x_i)^2}{2\lambda_i^{\alpha+1}+\eps^2}} \ls \prod_{i=1}^\infty \frac{\eps}{\sqrt{2\lambda_i^{\alpha+1}+\eps^2}}\rs. \] 

We are left to check that 
\[ \text{(a)}\, e^{-\sum_{i=1}^\infty \frac{\lambda_i^\alpha(m_i-x_i)^2}{2\lambda_i^{\alpha+1}+\eps^2}} \in(0,+\infty) \qquad \text{and} \qquad \text{(b)}\, \prod_{i=1}^\infty \frac{\eps}{\sqrt{2\lambda_i^{\alpha+1}+\eps^2}} \in(0,+\infty). \]

For (a),
\[ 0 \leq \sum_{i=1}^\infty \frac{\lambda_i^\alpha(m_i-x_i)^2}{2\lambda_i^{\alpha+1}+\eps^2} \leq \frac{1}{\eps^2} \|m_\mu - x\|_{X^\alpha}^2. \]
So,
\[ 0< e^{-\frac{1}{\eps^2} \|x-m\|_{X^\alpha}^2} \leq e^{-\sum_{i=1}^\infty \frac{\lambda_i^\alpha(m_i-x_i)^2}{2\lambda_i^{\alpha+1}+\eps^2}} \leq 1. \]

For (b), we have
\begin{align*}
0 & > \sum_{i=1}^\infty \log \frac{\eps}{\sqrt{2\lambda_i^{\alpha+1}+\eps^2}} \\
 & = \sum_{i=1}^\infty \l \log \eps - \frac12 \log\l 2\lambda_i^{\alpha+1}+\eps^2\r \r \\
 & \geq \sum_{i=1}^\infty \l \log \eps - \frac12 \l \log \eps^2 + \frac{2\lambda_i^{\alpha+1}}{\eps^2}\r\r \\
 & = - \frac{1}{\eps^2} \sum_{i=1}^\infty \lambda_i.
\end{align*}
So, 
\[ 1\geq \prod_{i=1}^\infty \frac{\eps}{\sqrt{2\lambda_i^{\alpha+1}+\eps^2}} \geq e^{-\frac{1}{\eps^2} \sum_{i=1}^\infty \lambda_i^{\alpha+1}} > 0. \qedhere \]
\end{proof}

We can apply the above theorem to a special case that we will use in the sequel.

\begin{lemma}
Let Assumptions~\ref{ass:Setting:X} and~\ref{ass:Setting:eta} hold.
Assume $\mu=\cN(m_\mu,C_\mu)$ and $\{(\lambda_i,e_i)\}_{i=1}^\infty$ are the eigenpairs of $C_\mu$.
For all $\alpha\in(-1,+\infty)$, $\eps >0$ and $\beta\in\bbR$, if $m_\mu\in X^\alpha$ and $\sum_{i=1}^\infty \lambda_i^{2(\alpha+1)}<+\infty$ then
\[ \int_X\frac{\int_X \eta\left(\frac{\|y - x\|_{X^\alpha}}{\varepsilon}\right) \lp\frac{\| y-x\|_{X^\beta}}{\eps}\rp^2 \, \dd\mu(y)}{ \int_X \eta\left(\frac{\|y -x\|_{X^\alpha}}{\varepsilon}\right)\,\dd\mu(y) }\,\dd\mu(x) = \sum_{i=1}^{\infty}\frac{2\lambda_i^{\beta+1} (\varepsilon^2 + \lambda_i^{\alpha+1})}{(\varepsilon^2 + 2 \lambda_i^{\alpha+1})^2}. \]
\end{lemma}

\begin{proof}
Since $\sum_{i=1}^\infty \lambda_i^{2(\alpha+1)}\leq \lambda_1^{\alpha+1} \sum_{i=1}^\infty \lambda_i^{\alpha+1}$, then  $\sum_{i=1}^\infty \lambda_i^{2(\alpha+1)}<+\infty$.
So by Proposition~\ref{prop:Back:Gauss:mu_support}, $\supp(\mu)\subset X^\alpha$.
Hence,
\[ \int_X\frac{\int_X \eta\left(\frac{\|y - x\|_{X^\alpha}}{\varepsilon}\right) \lp\frac{\| y-x\|_{X^\beta}}{\eps}\rp^2 \, \dd\mu(y)}{ \int_X \eta\left(\frac{\|y -x\|_{X^\alpha}}{\varepsilon}\right)\,\dd\mu(y) }\,\dd\mu(x) = \int_{X^\alpha} \frac{\int_X \eta\left(\frac{\|y - x\|_{X^\alpha}}{\varepsilon}\right) \lp\frac{\| y-x\|_{X^\beta}}{\eps}\rp^2 \, \dd\mu(y)}{ \int_X \eta\left(\frac{\|y -x\|_{X^\alpha}}{\varepsilon}\right)\,\dd\mu(y) }\,\dd\mu(x). \]
By applying Lemma~\ref{lem:Back:Tech:Conj} to the inner integral,
\[ \int_X\frac{\int_X \eta\left(\frac{\|y - x\|_{X^\alpha}}{\varepsilon}\right) \lp\frac{\| y-x\|_{X^\beta}}{\eps}\rp^2 \, \dd\mu(y)}{ \int_X \eta\left(\frac{\|y -x\|_{X^\alpha}}{\varepsilon}\right)\,\dd\mu(y) }\,\dd\mu(x) = \frac{1}{\eps^2} \int_{X^\alpha} \int_X \| x-y\|_{X^\beta}^2 \, \dd \mu_{x,\eps}(y) \, \dd \mu(x) \]
where $\mu_{x,\eps} = \cN(m_{x,\eps}, C_{x,\eps})$ and $m_{x,\eps}$ and $C_{x,\eps}$ are given by~\eqref{eq:Back:Tech:mxeps} and~\eqref{eq:Back:Tech:Cxeps}.

By Lemma~\ref{lem:Back:Gauss:KL_expansion}
\begin{align*}
& \int_X\frac{\int_X \eta\left(\frac{\|y - x\|_X}{\varepsilon}\right) \lp\frac{\| y-x\|_X}{\eps}\rp^2 \, \dd\mu(y)}{ \int_X \eta\left(\frac{\|y -x\|_X}{\varepsilon}\right)\,\dd\mu(y) }\,\dd\mu(x) \\ 
& \qquad \qquad = \frac{1}{\eps^2} \int_{X^\alpha} \int_X \sum_{i=1}^\infty \lambda_i^\beta (x_i-y_i)^2 \, \dd \mu_{x,\eps}(y) \, \dd \mu(x) \\
& \qquad \qquad = \frac{1}{\eps^2} \int_{X^\alpha} \sum_{i=1}^\infty \int_\bbR \lambda_i^\beta \l x_i - \eps \sqrt{\frac{\lambda_i}{\eps^2+2\lambda_i^{\alpha+1}}}\xi_i - \frac{\eps^2 m_i+2\lambda_i^{\alpha+1}x_i}{\eps^2+2\lambda_i^{\alpha+1}} \r^2 \rho(\xi_i) \, \dd \xi_i \, \dd \mu(x) \\
& \qquad \qquad = \frac{1}{\eps^2} \int_{X^\alpha} \sum_{i=1}^\infty \int_\bbR \lambda_i^\beta \l \frac{\eps^2(x_i-m_i)}{\eps^2+2\lambda_i^{\alpha+1}} - \eps \sqrt{\frac{\lambda_i}{\eps^2+2\lambda_i^{\alpha+1}}}\xi_i \r^2 \rho(\xi_i) \, \dd \xi_i \, \dd \mu(x) \\
& \qquad \qquad = \int_{X^\alpha} \sum_{i=1}^\infty \frac{\lambda_i^\beta}{\eps^2+2\lambda_i^{\alpha+1}} \int_\bbR \l \frac{\eps(x_i-m_i)}{\sqrt{\eps^2+2\lambda_i^{\alpha+1}}} - \sqrt{\lambda_i}\xi_i \r^2 \rho(\xi_i) \, \dd \xi_i \, \dd \mu(x) \\
& \qquad \qquad = \int_{X^\alpha} \sum_{i=1}^\infty \frac{\lambda_i^\beta}{\eps^2+2\lambda_i^{\alpha+1}} \int_\bbR \l \frac{\eps^2(x_i-m_i)^2}{\eps^2+2\lambda_i^{\alpha+1}} + \lambda_i\xi_i^2 - \frac{2\eps\sqrt{\lambda_i}\xi_i (x_i-m_i)}{\sqrt{\eps^2+2\lambda_i^{\alpha+1}}} \r \rho(\xi_i) \, \dd \xi_i \, \dd \mu(x) \\
& \qquad \qquad = \sum_{i=1}^\infty \frac{\lambda_i^\beta}{\eps^2+2\lambda_i^{\alpha+1}} \int_{X^\alpha} \l \frac{\eps^2(x_i-m_i)^2}{\eps^2+2\lambda_i^{\alpha+1}} + \lambda_i \r \, \dd \mu(x) \\
& \qquad \qquad = \sum_{i=1}^\infty \frac{\lambda_i^\beta}{\eps^2+2\lambda_i^{\alpha+1}} \l \frac{\eps^2\lambda_i}{\eps^2+2\lambda_i^{\alpha+1}} + \lambda_i \r \\
& \qquad \qquad = \sum_{i=1}^\infty \frac{2\lambda_i^{1+\beta}(\eps^2+\lambda_i^{\alpha+1})}{(\eps^2+2\lambda_i^{\alpha+1})^2}
\end{align*}
where we write $x_i$, $y_i$ and $m_i$ for $\langle x,e_i\rangle_X$, $\langle y,e_i\rangle_X$ and $\langle m_\mu,e_i\rangle_X$, $\rho$ is the standard Gaussian density on $\bbR$, and we use the fact that the eigenvalues for $C_{x,\mu}$ are $\frac{\eps^2\lambda_i}{\eps^2+2\lambda_i^{\alpha+1}}$. 
\end{proof}

\section{Asymptotics of Discrete and Nonlocal Functionals} \label{sec:proofs}

In this section, we present our main theoretical results for the asymptotic consistency of Theorem \ref{thm:Setting:PointCons}. We study the asymptotic behavior in two consecutive steps: first, the discrete to nonlocal continuum limit as $n \rightarrow \infty$ in Section \ref{subsec:proofs:LDC}, and the nonlocal to local continuum as $\varepsilon \rightarrow \infty$ in Section \ref{subsec:proofs:nonlocal}.

\subsection{Large Data Consistency Results} \label{subsec:proofs:LDC}

In this section, we study the asymptotic behavior in the large-sample limit. The results of interest are derived from the high-probability bound presented in Theorem \ref{thm:proofs:LDC:LDB} below.

\begin{theorem}
\label{thm:proofs:LDC:LDB}
Assume $\eta$ is bounded, $\mu\in \cP(X)$ and $u\in\Lp{\infty}(X)$.
Then, for any $\delta_1,\delta_2, \delta_3, \delta_4,\tau>0$ we have
\begin{align*}
& \left|\frac{1}{n}
\sum_{i=1}^{n}\frac{\sum_{j=1}^{n} \eta_{\varepsilon}\left(\left\|x_i-x_j\right\|_{X^\alpha}\right)|u(x_i)-u(x_j)|^2}{ \sum_{j=1}^{n} \eta_{\varepsilon}\left(\left\|x_i-x_j\right\|_{X^\alpha}\right)} - \int_X \frac{\int_X \eta_\varepsilon\left(\|x - y\|_{X^\alpha}\right)|u(x) - u(y)|^2 \, \dd\mu(y)}{ \int_X \eta_\varepsilon\left(\|x - y\|_{X^\alpha}\right)\,\dd\mu(y) }\,\dd\mu(x) \right| \notag \\
& \qquad \qquad \leq  \frac{\delta_1 + 4\|u \|_{\Lp{\infty}}^2 \delta_2}{\tau} + 4\|u\|_{\Lp{\infty}}^2\mu(D(X) < 2\tau) + \delta_3 + \delta_4,
\end{align*}
with probability at least $1- C(n,\delta_1,\delta_2,\delta_3,\delta_4,\tau)$, where 
\begin{align*}
C(n,\delta_1,\delta_2,\delta_3,\delta_4,\tau) & = 2n\exp\l-\frac{3n\delta_1^2}{8\|\eta\|_{\Lp{\infty}}\|u\|_{\Lp{\infty}}^2 \l 12\|\eta\|_{\Lp{\infty}}\|u\|_{\Lp{\infty}}^2 + \delta_1\r}\r \\
 & \qquad + 2n\exp\l-\frac{3n\delta_2^2}{2\|\eta\|_{\Lp{\infty}} \l 3\|\eta\|_{\Lp{\infty}} + \delta_2\r}\r \\
 & \qquad + 2n\exp\l-\frac{3n\tau^2}{2\|\eta\|_{\Lp{\infty}} \l 3\|\eta\|_{\Lp{\infty}} + \tau\r}\r \\
 & \qquad + \exp\l-\frac{3n\delta_3^2}{8 \|u\|_{\Lp{\infty}}^2 \l 12 \|u\|_{\Lp{\infty}}^2 + \delta_3\r}\r \\
 & \qquad + 2\exp\l-\frac{3n\delta_4^2}{8\|u\|_{\Lp{\infty}}^2 \l 12\|u\|_{\Lp{\infty}}^2 + \delta_4\r}\r
\end{align*}
and
\[ D(x) =\int_X \eta_\eps\left(\|x - y\|_{X^\alpha}\right)\, \dd \mu(y). \]
\end{theorem}

\begin{remark} \label{rem:proofs:LDC:thm2simp}
We can simplify the statement by assuming that $\delta_1=\delta_2=\tau^2$ and $\delta_3=\delta_4=\tau$ so that, for $0<\tau\leq 1$, 
\begin{align*}
& \left|\frac{1}{n}
\sum_{i=1}^{n}\frac{\sum_{j=1}^{n} \eta_{\varepsilon}\left(\left\|x_i-x_j\right\|_{X^\alpha}\right)|u(x_i)-u(x_j)|^2}{ \sum_{j=1}^{n} \eta_{\varepsilon}\left(\left\|x_i-x_j\right\|_{X^\alpha}\right)} - \int_X \frac{\int_X \eta_\varepsilon\left(\|x - y\|_{X^\alpha}\right)|u(x) - u(y)|^2 \, \dd\mu(y)}{ \int_X \eta_\varepsilon\left(\|x - y\|_{X^\alpha}\right)\,\dd\mu(y) }\,\dd\mu(x) \right| \\
& \qquad \qquad \leq \hat{C} \l\tau + \mu(D(X)<2\tau)\r
\end{align*}
with probability at least $1-Cne^{-cn\tau^4}$ where $\hat{C},C,c>0$ are independent of $n$, $\eps$ and $\tau$.
\end{remark}

\begin{proof}[Proof of Theorem~\ref{thm:proofs:LDC:LDB}.]
For simplicity, we introduce the following notation: 
\begin{align*}
N_n(x_i) & :=1/n\sum_{j=1}^{n} \eta_{\varepsilon}\left(\left\|x_i-x_j\right\|_{X^\alpha}\right)|u(x_i)-u(x_j)|^2, \\
D_n(x_i) & :=1/n\sum_{j=1}^{n} \eta_{\varepsilon}\left(\left\|x_i-x_j\right\|_{X^\alpha}\right), \\
N(x) & :=\int_X \eta_{\varepsilon}\left(\|x - y\|_{X^\alpha}\right)|u(x) - u(y)|^2 \, \dd \mu(y), \text{ and } \\
D(x) & :=\int_X \eta_\eps\left(\|x - y\|_{X^\alpha}\right)\, \dd \mu(y).
\end{align*}
In this notation the theorem states
\[ \la \frac{1}{n} \sum_{i=1}^n \frac{N_n(x_i)}{D_n(x_i)} - \int_X \frac{N(x)}{D(x)} \, \dd \mu(x) \ra \leq  \frac{\delta_1 + 4\|u \|_{\Lp{\infty}}^2 \delta_2}{\tau} + 4\|u\|_{\Lp{\infty}}^2\mu(D(X) < 2\tau) + \delta_3 + \delta_4 \]
with probability at least $1-C(n,\delta_1,\delta_2,\delta_3,\delta_4,\tau)$.

By the triangle inequality we have
\begin{equation} \label{eq:proofs:LDC:NonLocalDisCtsTriIneq}
\begin{aligned}
\la \frac{1}{n} \sum_{i=1}^n \frac{N_n(x_i)}{D_n(x_i)} - \int_X \frac{N(x)}{D(x)} \, \dd \mu(x) \ra & \leq \underbrace{\la \frac{1}{n} \sum_{i=1}^n \frac{N_n(x_i)}{D_n(x_i)} - \frac{1}{n} \sum_{i=1}^n \frac{N(x_i)}{D(x_i)} \ra}_{=:\RomanI_n} \\
 & \qquad \qquad + \underbrace{\la \frac{1}{n} \sum_{i=1}^n \frac{N(x_i)}{D(x_i)} - \int_X \frac{N(x)}{D(x)} \, \dd \mu(x) \ra}_{=:\RomanII_n}
\end{aligned}
\end{equation}

We first estimate $\RomanI_n$. 
For a fixed $x_i$ and a fixed $\tau > 0$, we can write 
\begin{equation} 
\begin{aligned}
\left| \frac{N_n(x_i)}{D_n(x_i) } - \frac{N(x_i) }{D(x_i) } \right| & = \one_{E_\tau} \left|\frac{N_n(x_i)}{D_n(x_i) } - \frac{N(x_i) }{D(x_i) } \right| + \one_{E_\tau^c} \left|\frac{N_n(x_i)}{D_n(x_i) } - \frac{N(x_i) }{D(x_i) } \right|
\end{aligned}
\end{equation}
where $E_\tau = \{ \min\left(D_n(x_i), D(x_i)\right) < \tau \}$.
This allows us to study the events with small and large denominators separately. 
When $E_\tau^c$ holds ($\min\left(D_n(x_i), D(x_i)\right) \ge \tau$) we have  
$$
\begin{aligned}
\one_{E_\tau^c} \left|\frac{N_n(x_i)}{D_n(x_i) } - \frac{N(x_i) }{D(x_i) } \right|
    & = \one_{E_\tau^c} \left|\frac{N_n(x_i)}{D_n(x_i) } -\frac{N(x_i)}{D_n(x_i) } +\frac{N(x_i)}{D_n(x_i) }  - \frac{N(x_i) }{D(x_i) } \right| \\
    & = \one_{E_\tau^c} \left|\frac{N_n(x_i)-N(x_i)}{D_n(x_i) } + \frac{N(x_i)(D(x_i) - D_n(x_i)) }{D_n(x_i) D(x_i)} \right|.
\end{aligned}
$$
Since $u$ is bounded, we have $N(x_i) \le 4\| u\|_{\Lp{\infty}}^2 D(x_i)$ (and $N_n(x_i) \le 4\| u\|_{\Lp{\infty}}^2 D_n(x_i)$), which implies 
$$
\begin{aligned}
\one_{E_\tau^c} \left|\frac{N_n(x_i)}{D_n(x_i) } - \frac{N(x_i) }{D(x_i) } \right|
 \le   \left|\frac{N_n(x_i)-N(x_i)}{\tau } \right| + \frac{4\|u \|_{\Lp{\infty}}^2\left|D(x_i) - D_n(x_i)\right| }{\tau}. 
\end{aligned}
$$
Additionally, as $|N_n(x_i)-N(x_i)| \le 4\|\eta\|_{\Lp{\infty}}\|u\|_{\Lp{\infty}}^2$ and $|D_n(x_i)-D(x_i)| \le \|\eta\|_{\Lp{\infty}}$ by Bernstein's inequality \cite[Proposition 2.10]{wainwright2019high}, we have
$$ \bbP\left( |N_n(x_i)-N(x_i)| \le \delta_1\right) \geq 1- 2\exp\left(-\frac{n\delta_1^2}{32\|\eta\|_{\Lp{\infty}}^2\|u\|_{\Lp{\infty}}^4+8/3\|\eta\|_{\Lp{\infty}}\|u\|_{\Lp{\infty}}^2 \delta_1}\right). $$ 
Similarly,
\begin{equation} \label{eq:proofs:LDC:DDnBernsteinBound}
\bbP\left( |D_n(x_i)-D(x_i)| \le \delta_2\right) \geq 1- 2\exp\left(-\frac{n\delta_2^2}{2\|\eta\|_{\Lp{\infty}}^2+2/3\|\eta\|_{\Lp{\infty}} \delta_2}\right).
\end{equation}
Hence,  
\begin{equation} \label{eq:proofs:LDC:LDBthm_est_1a}
\begin{aligned}
& \bbP\left(\one_{E_\tau^c} \left|\frac{N_n(x_i)}{D_n(x_i) } - \frac{N(x_i) }{D(x_i) } \right| \le \frac{\delta_1 + 4\|u \|_{\Lp{\infty}}^2 \delta_2}{\tau}\right) \\
& \qquad \qquad \geq \bbP\l \left|\frac{N_n(x_i)-N(x_i)}{\tau } \right| \leq \frac{\delta_1}{\tau} \text{ and } \frac{4\|u \|_{\Lp{\infty}}^2\left|D(x_i) - D_n(x_i)\right| }{\tau} \leq \frac{4\|u\|_{\Lp{\infty}}^2 \delta_2}{\tau} \r \\
& \qquad \qquad \geq \bbP\l \left|N_n(x_i)-N(x_i) \right| \leq \delta_1\r + \bbP\l \left|D(x_i) - D_n(x_i)\right| \leq  \delta_2 \r - 1 \\
& \qquad \qquad \geq 1 - 2\exp\left(-\frac{n\delta_1^2}{32\|\eta\|_{\Lp{\infty}}^2\|u\|_{\Lp{\infty}}^4+8/3\|\eta\|_{\Lp{\infty}}\|u\|_{\Lp{\infty}}^2 \delta_1}\right) \\
& \qquad \qquad \qquad \qquad - 2\exp\left(-\frac{n\delta_2^2}{2\|\eta\|_{\Lp{\infty}}^2+2/3\|\eta\|_{\Lp{\infty}} \delta_2}\right).
\end{aligned}
\end{equation}

Next, we have 
\[
\begin{aligned}
\one_{E_\tau} \left|\frac{N_n(x_i)}{D_n(x_i) } - \frac{N(x_i) }{D(x_i) } \right|  \le 4\one_{E_\tau}  \|u\|_{\Lp{\infty}}^2 \leq 4 \|u\|_{\Lp{\infty}}^2.
\end{aligned}
\]
If $D(x_i) <2\tau$ then
\[ \bbP\left( \one_{E_\tau} \left|\frac{N_n(x_i)}{D_n(x_i) } - \frac{N(x_i) }{D(x_i) } \right| \leq 4 \underbrace{\one_{\{D(x_i)<2\tau\}}}_{=1} \|u\|^2_{\Lp{\infty}}  \right) = 1. \]
If $D(x_i)\geq 2\tau$, then by~\eqref{eq:proofs:LDC:DDnBernsteinBound} with $\delta_2 = \tau$,
\[
\begin{aligned}
& \bbP\left( \one_{E_\tau} \left|\frac{N_n(x_i)}{D_n(x_i) } - \frac{N(x_i) }{D(x_i) } \right| \leq 4 \underbrace{\one_{\{D(x_i)<2\tau\}}}_{=0} \|u\|^2_{\Lp{\infty}}  \right) \\
& \qquad = \underbrace{\bbP\left( \one_{E_\tau} \left|\frac{N_n(x_i)}{D_n(x_i) } - \frac{N(x_i) }{D(x_i) } \right| = 0 \bigg| E_\tau \right)}_{\geq 0} \bbP(E_\tau) \\
& \qquad \qquad \qquad + \underbrace{\bbP\left( \one_{E_\tau} \left|\frac{N_n(x_i)}{D_n(x_i) } - \frac{N(x_i) }{D(x_i) } \right| = 0 \bigg| E_\tau^{\rmc} \right)}_{=1} \bbP(E_\tau^{\rmc}) \\
& \qquad \geq \bbP\left( \{ \min\left(D_n(x_i), D(x_i)\right) \geq \tau \} \right) \\
& \qquad = \bbP\left( \{ D_n(x_i)\geq \tau\}\right) \\
& \qquad \geq \bbP\left( \left| D_n(x_i) - D(x_i)\right| \leq \tau\right) \\
& \qquad \geq 1 - 2\exp\left(-\frac{n\tau^2}{2\|\eta\|_{\Lp{\infty}}^2 + \frac{2}{3} \|\eta\|_{\Lp{\infty}}\tau}\right).
\end{aligned}
\]
Hence,  
\begin{equation} \label{eq:proofs:LDC:LDBthm_est_1b}
\begin{aligned}
& \bbP\left(\one_{E_\tau} \left|\frac{N_n(x_i)}{D_n(x_i) } - \frac{N(x_i) }{D(x_i) } \right| \leq 4 \one_{\{D(x_i)<2\tau\}} \|u\|^2_{\Lp{\infty}}\right) \\
& \qquad \qquad \geq 1 - 2 \one_{\{D(x_i)>2\tau\}} \exp\left(-\frac{n\tau^2}{2\|\eta\|_{\Lp{\infty}}^2+2/3\|\eta\|_{\Lp{\infty}}\tau}\right) \\
& \qquad \qquad \geq 1 - 2\exp\left(-\frac{n\tau^2}{2\|\eta\|_{\Lp{\infty}}^2+2/3\|\eta\|_{\Lp{\infty}} \tau}\right).
\end{aligned}
\end{equation}

Putting \eqref{eq:proofs:LDC:LDBthm_est_1a} and \eqref{eq:proofs:LDC:LDBthm_est_1b} together, we have 
\begin{equation} \label{eq:proofs:LDC:LDBthm_upperbound}
\begin{aligned}
& \bbP\l \la\frac{N_n(x_i)}{D_n(x_i)} - \frac{N(x_i)}{D(x_i)}\ra \leq \frac{\delta_1 + 4\|u\|_{\Lp{\infty}}^2\delta_2}{\tau} + 4 \one_{\{D(x_i)<2\tau\}} \|u\|_{\Lp{\infty}}^2 \r \\
& \qquad \geq \bbP\Bigg( \one_{E_\tau}\la\frac{N_n(x_i)}{D_n(x_i)} - \frac{N(x_i)}{D(x_i)}\ra \leq 4 \one_{\{D(x_i)<2\tau\}} \|u\|_{\Lp{\infty}}^2 \\
& \qquad   \qquad \text{and } \one_{E_\tau^c}\la\frac{N_n(x_i)}{D_n(x_i)} - \frac{N(x_i)}{D(x_i)}\ra \leq \frac{\delta_1 + 4\|u\|_{\Lp{\infty}}^2\delta_2}{\tau} \Bigg) \\
& \qquad \geq \bbP\left(\one_{E_\tau^c} \left|\frac{N_n(x_i)}{D_n(x_i) } - \frac{N(x_i) }{D(x_i) } \right| \le \frac{\delta_1 + 4\|u \|_{\Lp{\infty}}^2 \delta_2}{\tau}\right) \\
& \qquad  \qquad + \bbP\left(\one_{E_\tau} \left|\frac{N_n(x_i)}{D_n(x_i) } - \frac{N(x_i) }{D(x_i) } \right| \leq 4 \one_{\{D(x_i)<2\tau\}} \|u\|^2_{\Lp{\infty}}\right) - 1 \\
& \qquad \geq 1 - 2\exp\left(-\frac{n\delta_1^2}{32\|\eta\|_{\Lp{\infty}}^2\|u\|_{\Lp{\infty}}^4+8/3\|\eta\|_{\Lp{\infty}}\|u\|_{\Lp{\infty}}^2 \delta_1}\right) \\
& \qquad \qquad - 2\exp\left(-\frac{n\delta_2^2}{2\|\eta\|_{\Lp{\infty}}^2+2/3\|\eta\|_{\Lp{\infty}} \delta_2}\right) - 2\exp\left(-\frac{n\tau^2}{2\|\eta\|_{\Lp{\infty}}^2+2/3\|\eta\|_{\Lp{\infty}} \tau}\right).
\end{aligned}
\end{equation}

Applying inequality \eqref{eq:proofs:LDC:LDBthm_upperbound} to all $i = 1, \dots, n$ we have 
\[ 
\begin{aligned}
& \bbP\left( \left|\frac{1}{n} \sum_{i=1}^n \left(\frac{N_n(x_i)}{D_n(x_i)} - \frac{N(x_i)}{D(x_i)} \right) \right| \leq \frac{\delta_1 + 4\|u\|_{\Lp{\infty}}^2\delta_2}{\tau} + \frac{4 \|u\|_{\Lp{\infty}}^2}{n} \sum_{i=1}^n \one_{\{D(x_i)<2\tau\}} \right) \\
& \qquad \geq 1 - n + \sum_{i=1}^n \bbP\left( \left|\frac{N_n(x_i)}{D_n(x_i)} - \frac{N(x_i)}{D(x_i)} \right| \leq \frac{\delta_1 + 4\|u\|_{\Lp{\infty}}^2\delta_2}{\tau} + 4 \|u\|_{\Lp{\infty}}^2 \one_{\{D(x_i)<2\tau\}} \right) \\
& \qquad \geq 1 - 2n\exp\left(-\frac{n\delta_1^2}{32\|\eta\|_{\Lp{\infty}}^2\|u\|_{\Lp{\infty}}^4+8/3\|\eta\|_{\Lp{\infty}}\|u\|_{\Lp{\infty}}^2 \delta_1}\right) \\
& \qquad \qquad \qquad - 2n\exp\left(-\frac{n\delta_2^2}{2\|\eta\|_{\Lp{\infty}}^2+2/3\|\eta\|_{\Lp{\infty}} \delta_2}\right) - 2n\exp\left(-\frac{n\tau^2}{2\|\eta\|_{\Lp{\infty}}^2+2/3\|\eta\|_{\Lp{\infty}} \tau}\right).
\end{aligned}
\]

Using the same one-sided Bernstein's inequality~(\cite[Proposition 2.14]{wainwright2019high}), we have
$$
\bbP\left(\frac{4\|u\|_{\Lp{\infty}}^2}{n} \sum_{i=1}^{n}\one_{\left\{D(x_i) < 2\tau\right\}} - 4\|u\|_{\Lp{\infty}}^2 \mu(D(X)<2\tau) \le \delta_3 \right) \geq 1 - \exp\left( -\frac{n\delta_3^2}{32 \|u\|^4_{\Lp{\infty}} + \frac{8}{3} \|u\|_{\Lp{\infty}}^2 \delta_3} \right).
$$
Hence, 
\begin{align}
& \bbP\left( \left|\frac{1}{n} \sum_{i=1}^n \left(\frac{N_n(x_i)}{D_n(x_i)} - \frac{N(x_i)}{D(x_i)} \right) \right| \leq \frac{\delta_1 + 4\|u\|_{\Lp{\infty}}^2\delta_2}{\tau} + 4 \|u\|_{\Lp{\infty}}^2 \mu(D(X)<2\tau) + \delta_3 \right) \notag \\
& \qquad \geq \bbP\left( \left|\frac{1}{n} \sum_{i=1}^n \left(\frac{N_n(x_i)}{D_n(x_i)} - \frac{N(x_i)}{D(x_i)} \right) \right| \leq \frac{\delta_1 + 4\|u\|_{\Lp{\infty}}^2\delta_2}{\tau} + \frac{4 \|u\|_{\Lp{\infty}}^2}{n} \sum_{i=1}^n \one_{\{D(x_i)<2\tau\}} \right) \notag \\
& \qquad \qquad \qquad + \bbP\left( \frac{4\|u\|_{\Lp{\infty}}^2}{n} \sum_{i=1}^{n}\one_{\left\{D(x_i) < 2\tau\right\}} - 4\|u\|_{\Lp{\infty}}^2 \mu(D(X)<2\tau) \le \delta_3 \right) - 1 \notag \\
& \qquad \geq 1 - 2n\exp\left(-\frac{n\delta_1^2}{32\|\eta\|_{\Lp{\infty}}^2\|u\|_{\Lp{\infty}}^4+8/3\|\eta\|_{\Lp{\infty}}\|u\|_{\Lp{\infty}}^2 \delta_1}\right) \notag \\
& \qquad \qquad \qquad - 2n\exp\left(-\frac{n\delta_2^2}{2\|\eta\|_{\Lp{\infty}}^2+2/3\|\eta\|_{\Lp{\infty}} \delta_2}\right) - 2n\exp\left(-\frac{n\tau^2}{2\|\eta\|_{\Lp{\infty}}^2+2/3\|\eta\|_{\Lp{\infty}} \tau}\right) \label{eq:proofs:LDC:In} \\
& \qquad \qquad \qquad -  \exp\left( -\frac{n\delta_3^2}{32 \|u\|^4_{\Lp{\infty}} + \frac{8}{3} \|u\|_{\Lp{\infty}}^2 \delta_3} \right). \notag
\end{align}

Next, we apply Bernstein's inequality to the second term, $\RomanII_n$, in~\eqref{eq:proofs:LDC:NonLocalDisCtsTriIneq}. 
Recalling $N(x)\leq 4\|u\|_{\Lp{\infty}}^2 D(x)$, we have
\begin{equation} \label{eq:proofs:LDC:IIn}
\bbP\l\la \frac{1}{n} \sum_{i=1}^n \frac{N(x_i)}{D(x_i)} - \int_X \frac{N(x)}{D(x)} \, \dd \mu(x) \ra \leq \delta_4 \r \geq 1 - 2\exp\l -\frac{n\delta_4^2}{32\|u\|_{\Lp{\infty}}^2 + 8/3\|u\|_{\Lp{\infty}}^2\delta_4}\r.
\end{equation}

Combining~\eqref{eq:proofs:LDC:In} and~\eqref{eq:proofs:LDC:IIn} concludes the proof.
\end{proof}

The $n\to\infty$ is derived as a simple consequence of the above theorem.

\begin{corollary}
\label{cor:proofs:LDC:LDC}
Let the assumptions in Theorem~\ref{thm:proofs:LDC:LDB} be satisfied.  
Then with probability $1$ we have
\begin{align*}
& \lim_{n\rightarrow \infty}\frac{1}{n}
\sum_{i=1}^{n}\frac{\sum_{j=1}^{n} \eta_{\varepsilon}\left(\left\|x_i-x_j\right\|_{X^\alpha}\right)|u(x_i)-u(x_j)|^2}{ \sum_{j=1}^{n} \eta_{\varepsilon}\left(\left\|x_i-x_j\right\|_{X^\alpha}\right)} \\
& \qquad \qquad = \int_X \frac{\int_X \eta_\varepsilon\left(\|x - y\|_{X^\alpha}\right)|u(x) - u(y)|^2 \, \dd\mu(y)}{ \int_X \eta_\varepsilon\left(\|x - y\|_{X^\alpha}\right)\,\dd\mu(y) }\,\dd\mu(x).
\end{align*}
\end{corollary}

\begin{proof}
Following the simplification in Remark~\ref{rem:proofs:LDC:thm2simp} we have
\begin{align*}
& \left|\frac{1}{n}\sum_{i=1}^{n}\frac{\sum_{j=1}^{n} \eta_{\varepsilon}\left(\left\|x_i-x_j\right\|_{X^\alpha}\right)|u(x_i)-u(x_j)|^2}{ \sum_{j=1}^{n} \eta_{\varepsilon}\left(\left\|x_i-x_j\right\|_{X^\alpha}\right)} - \int_X \frac{\int_X \eta_\varepsilon\left(\|x - y\|_{X^\alpha}\right)|u(x) - u(y)|^2 \, \dd\mu(y)}{ \int_X \eta_\varepsilon\left(\|x - y\|_{X^\alpha}\right)\,\dd\mu(y) }\,\dd\mu(x) \right| \\
& \qquad \qquad \qquad \qquad \geq \hat{C}\l \tau + \mu(D(X)<2\tau) \r 
\end{align*}
with probability at most $Cne^{-cn\tau^4}$.
We can bound the probability by
\[ Cne^{-cn\tau^4} \leq Ce^{\log n - cn\tau^4} = Ce^{-\log n\l \frac{cn\tau^4}{\log n} - 1\r}. \]
We choose $\tau = \tau_n = \sqrt[4]{\frac{3\log n}{cn}}$ so that $e^{-\log n\l \frac{cn\tau^4}{\log n} - 1\r} = n^{-2}$.
By the Borel-Cantelli lemma there exists $N$ such that for all $n\geq N$
\begin{align*}
& \left|\frac{1}{n}\sum_{i=1}^{n}\frac{\sum_{j=1}^{n} \eta_{\varepsilon}\left(\left\|x_i-x_j\right\|_{X^\alpha}\right)|u(x_i)-u(x_j)|^2}{ \sum_{j=1}^{n} \eta_{\varepsilon}\left(\left\|x_i-x_j\right\|_{X^\alpha}\right)} - \int_X \frac{\int_X \eta_\varepsilon\left(\|x - y\|_{X^\alpha}\right)|u(x) - u(y)|^2 \, \dd\mu(y)}{ \int_X \eta_\varepsilon\left(\|x - y\|_{X^\alpha}\right)\,\dd\mu(y) }\,\dd\mu(x) \right| \\
& \qquad \qquad \qquad \qquad \leq \hat{C}\l \tau_n + \mu(D(X)<2\tau_n) \r
\end{align*}
and with probability one $N<\infty$.

By the dominated convergence theorem $\mu(D(X)<2\tau_n) \to 0$ (since $D(x)>0$ for every $x\in X$).
\end{proof}

\subsection{Continuum Localization Limits} \label{subsec:proofs:nonlocal}

In this section, we analyze the asymptotic behavior in the transition from the nonlocal continuum to the local continuum. The limit is established in two steps: first by performing a first-order Taylor expansion of the nonlocal objective function, and then by proving the convergence of this first-order approximation.

The following technical lemma will be used in Lemma \ref{lem:proofs:nonlocal:nonloc_der} to show that the Taylor expansion error can be controlled and vanishes asymptotically.

\begin{lemma}
\label{lem:proofs:nonlocal:4thmoment}
Suppose Assumptions~\ref{ass:Setting:X}-\ref{ass:Setting:eta} are satisfied.
Let $\alpha\in (-1,+\infty)$ and $\beta\in\bbR$.
Assume $\sum_{i=1}^\infty \lambda_i^{\beta+\frac12 - \frac{\alpha}{2}}<+\infty$, $\sum_{i=1}^\infty \lambda_i^{\alpha+1}<+\infty$ and $m_\mu\in X^\alpha$.
Then,
\[ \lim_{\eps \to 0} \frac{1}{\eps^2} \int_X \frac{\int_X \eta\l\frac{\|y - x\|_{X^\alpha}}{\eps}\r \| y-x\|_{X^\beta}^4 \, \dd\mu(y)}{ \int_X \eta\l\frac{\|y -x\|_{X^\alpha}}{\eps}\r\,\dd\mu(y) } \, \dd \mu(x)  = 0. \]
\end{lemma}

\begin{proof}
By Lemma \ref{lem:Back:Tech:Conj}, we know that for $x\in X^\alpha$
\begin{align*}
\frac{\int_X \eta\left(\frac{\|y - x\|_{X^\alpha}}{\varepsilon}\right) \| y-x\|_{X^\beta}^4 \, \dd\mu(y)}{ \int_X \eta\left(\frac{\|y -x\|_{X^\alpha}}{\varepsilon}\right)\,\dd\mu(y) } = \int_X \|y-x \|_{X^\beta}^4 \,\dd \mu_{x,\varepsilon}(y), 
\end{align*}
where $\mu_{x,\varepsilon} = \mathcal{N}(m_{x,\eps},C_{x,\varepsilon})$ with $m_{x,\eps}$ and $C_{x,\eps}$ defined by~\eqref{eq:Back:Tech:mxeps} and~\eqref{eq:Back:Tech:Cxeps} respectively. 
We let $\bar{\lambda}_i = \frac{\eps^2 \lambda_i}{\eps^2+2\lambda_i^{\alpha+1}}$ be the eigenvalues of $C_{x,\eps}$ and $\bar{m}_i = \langle m_{x,\eps},e_i\rangle_X = \frac{\eps^2 m_i + 2\lambda_i^{\alpha+1}x_i}{\eps^2 + 2\lambda_i^{\alpha+1}}$ where $m_i = \langle m_\mu,e_i\rangle_X$.
By Proposition~\ref{prop:Back:Gauss:mu_support}, $\supp(\mu)\subset X^\alpha$.
Using the definition of the inner product in $X^\beta$ and Lemma~\ref{lem:Back:Gauss:KL_expansion} we can compute:
\begin{align*}
& \int_X \frac{\int_X \eta\left(\frac{\|y - x\|_{X^\alpha}}{\varepsilon}\right) \| y-x\|_{X^\beta}^4 \, \dd\mu(y)}{ \int_X \eta\left(\frac{\|y -x\|_{X^\alpha}}{\varepsilon}\right)\,\dd\mu(y) } \, \dd \mu(x) \\
& \,\, = \int_{X^\alpha} \int_X \|y-x\|_{X^\beta}^4 \, \dd \mu_{x,\eps}(y) \, \dd \mu(x) \\
& \,\, = \int_{X^\alpha} \int_X \l \sum_{i=1}^\infty \lambda_i^\beta (y_i-x_i)^2\r^2 \, \dd \mu_{x,\eps}(y) \, \dd \mu(x) \\
& \,\, = \int_{X^\alpha} \sum_{i\neq j} \int_X \lambda_i^\beta \lambda_j^\beta (y_i-x_i)^2(y_j-x_j)^2 \, \dd \mu_{x,\eps}(y) \, \dd \mu(x) + \int_{X^\alpha} \sum_{i=1}^\infty \int_X \lambda_i^{2\beta} (y_i-x_i)^4 \, \dd \mu_{x,\eps}(y) \, \dd \mu(x) \\
& \,\, = \int_{X^\alpha} \sum_{i\neq j} \int_{\bbR^2} \lambda_i^\beta \lambda_j^\beta \l \sqrt{\bar{\lambda}_i} \xi_i + \bar{m}_i - x_i \r^2 \l \sqrt{\bar{\lambda}_j} \xi_j + \bar{m}_j - x_j \r^2 \rho(\xi_i) \rho(\xi_j) \, \dd \xi_i \, \dd \xi_j \, \dd \mu(x) \\
& \qquad \,\, + \int_{X^\alpha} \sum_{i=1}^\infty \int_{\bbR} \lambda_i^{2\beta} \l \sqrt{\bar{\lambda}_i} \xi_i + \bar{m}_i - x_i\r^4 \rho(\xi_i) \, \dd \xi_i \, \dd \mu(x) \\
& \,\, = \int_{X^\alpha} \sum_{i\neq j} \lambda_i^\beta \lambda_j^\beta \ls \int_{\bbR} \ls \bar{\lambda}_i \xi_i^2 + (\bar{m}_i - x_i)^2 \rs \rho(\xi_i) \, \dd \xi_i \int_\bbR \ls \bar{\lambda}_j \xi_j^2 + (\bar{m}_j - x_j)^2 \rs \rho(\xi_j) \, \dd \xi_j \rs \, \dd \mu(x) \\
& \qquad \,\, + \int_{X^\alpha} \sum_{i=1}^\infty \lambda_i^{2\beta} \int_{\bbR} \ls \bar{\lambda}_i^2 \xi_i^4 + 6\bar{\lambda}_i \xi_i^2 (\bar{m}_i-x_i)^2 + (\bar{m}_i-x_i)^4 \rs \rho(x_i) \, \dd \xi_i \, \dd \mu(x) \\
& \,\, = \int_{X^\alpha}  \sum_{i\neq j} \lambda_i^\beta\lambda_j^\beta \ls \bar{\lambda}_i + (\bar{m}_i-x_i)^2\rs \ls \bar{\lambda}_j + (\bar{m}_j-x_j)^2\rs \, \dd \mu(x) \\
& \qquad \,\, + \int_{X^\alpha} \sum_{i=1}^\infty \lambda_i^{2\beta} \ls 3\bar{\lambda}_i^2 + 6 \bar{\lambda}_i(\bar{m}_i-x_i)^2 + (\bar{m}_i-x_i)^4 \rs \, \dd \mu(x) \\
& \,\, = \sum_{i\neq j} \lambda_i^\beta \lambda_j^\beta \ls \frac{\eps^2\lambda_i}{\eps^2+2\lambda_i^{\alpha+1}} + \frac{\eps^4}{(\eps^2+2\lambda_i^{\alpha+1})^2} \int_{\bbR} \lambda_i \xi_i^2 \rho(\xi_i) \, \dd \xi_i \rs \\
 & \qquad \,\, \times \ls \frac{\eps^2\lambda_j}{\eps^2+2\lambda_j^{\alpha+1}} + \frac{\eps^4}{(\eps^2+2\lambda_j^{\alpha+1})^2} \int_{\bbR} \lambda_j \xi_j^2 \rho(\xi_j) \, \dd \xi_j \rs \\
& \qquad \,\, + 3\eps^4 \sum_{i=1}^\infty \frac{\lambda_i^{2+2\beta}}{(\eps^2+2\lambda_i^{\alpha+1})^2} + 6\eps^6 \sum_{i=1}^\infty \frac{\lambda_i^{2\beta+1}}{(\eps^2+2\lambda_i^{\alpha+1})^3} \int_{X^\alpha} (m_i-x_i)^2 \, \dd \mu(x) \\
& \qquad \,\, + \eps^8 \sum_{i=1}^\infty \frac{\lambda_i^{2\beta}}{(\eps^2+2\lambda_i^{\alpha+1})^4} \int_{X^\alpha} (m_i-x_i)^4 \, \dd \mu(x) \\
& \,\, = \sum_{i\neq j} \lambda_i^\beta \lambda_j^\beta \ls \frac{\eps^2\lambda_i}{\eps^2+2\lambda_i^{\alpha+1}} + \frac{\eps^4\lambda_i}{(\eps^2+2\lambda_i^{\alpha+1})^2} \rs \ls \frac{\eps^2\lambda_j}{\eps^2+2\lambda_j^{\alpha+1}} + \frac{\eps^4\lambda_j}{(\eps^2+2\lambda_j^{\alpha+1})^2} \rs \\
& \qquad \,\, + 3\eps^4 \sum_{i=1}^\infty \frac{\lambda_i^{2+2\beta}}{(\eps^2+2\lambda_i^{\alpha+1})^2} + 6\eps^6 \sum_{i=1}^\infty \frac{\lambda_i^{2\beta+2}}{(\eps^2+2\lambda_i^{\alpha+1})^3} + 3\eps^8 \sum_{i=1}^\infty \frac{\lambda_i^{2\beta+2}}{(\eps^2+2\lambda_i^{\alpha+1})^4} \\
& \,\, = \lp\sum_{i =1 }^{\infty} \lambda_i^\beta \ls \frac{\eps^2\lambda_i}{\eps^2+2\lambda_i^{\alpha+1}} + \frac{\eps^4\lambda_i}{(\eps^2+2\lambda_i^{\alpha+1})^2} \rs \rp^2+ 2 \sum_{i=1}^{\infty} \lp \lambda_i^\beta \ls \frac{\eps^2\lambda_i}{\eps^2+2\lambda_i^{\alpha+1}} + \frac{\eps^4\lambda_i}{(\eps^2+2\lambda_i^{\alpha+1})^2} \rs \rp^2
\end{align*}
where we use the shorthand notation $x_i = \langle x,e_i\rangle_X$ and $y_i = \langle y,e_i\rangle_X$, and $\rho$ is the density of the standard Gaussian on $\bbR$.

Then we have 
\begin{align*}
& \frac{1}{\varepsilon^2} \int_X \frac{\int_X \eta\left(\frac{\|y - x\|_{X^\alpha}}{\varepsilon}\right) \| y-x\|_{X^\beta}^4 \, \dd\mu(y)}{ \int_X \eta\left(\frac{\|y -x\|_{X^\alpha}}{\varepsilon}\right)\,\dd\mu(y) } \, \dd \mu(x)  \\
& \,\, \le \frac{1}{\varepsilon^2}  \lp\sum_{i =1 }^{\infty} \lambda_i^\beta \ls \frac{\eps^2\lambda_i}{\eps^2+2\lambda_i^{\alpha+1}} + \frac{\eps^4\lambda_i}{(\eps^2+2\lambda_i^{\alpha+1})^2} \rs \rp^2+ \frac{2}{\varepsilon^2}   \sum_{i=1}^{\infty} \lp \lambda_i^\beta \ls \frac{\eps^2\lambda_i}{\eps^2+2\lambda_i^{\alpha+1}} + \frac{\eps^4\lambda_i}{(\eps^2+2\lambda_i^{\alpha+1})^2} \rs \rp^2\\
& \,\, \le \frac{3}{\varepsilon^2}  \lp\sum_{i =1 }^{\infty} \lambda_i^\beta \ls \frac{\eps^2\lambda_i}{\eps^2+2\lambda_i^{\alpha+1}} + \frac{\eps^4\lambda_i}{(\eps^2+2\lambda_i^{\alpha+1})^2} \rs \rp^2 \\
& \,\, = 3 \l \underbrace{\sum_{i=1}^\infty \frac{\eps \lambda_i^{\beta+1}}{\eps^2+2\lambda_i^{\alpha +1}}}_{=:\RomanIII_\eps} + \underbrace{\sum_{i=1}^\infty \frac{\eps^3 \lambda_i^{\beta+1}}{(\eps^2+2\lambda_i^{\alpha +1})^2}}_{=:\RomanIV_\eps} \r^2
\end{align*}

For $\RomanIV_\eps$ we observe that
\[ \sum_{i=1}^\infty \frac{\eps^3 \lambda_i^{\beta+1}}{(\eps^2+2\lambda_i^{\alpha +1})^2} \leq \sum_{i=1}^\infty \frac{\eps \lambda_i^{\beta+1}}{\eps^2+2\lambda_i^{\alpha +1}} \]
since $\eps^2\leq \eps^2+2\lambda_i^{\alpha+1}$.
So $\RomanIV_\eps\leq \RomanIII_\eps$.

For $\RomanIII_\eps$ one can easily check that
\[ \frac{\eps \lambda_i^{\beta+1}}{\eps^2+2\lambda_i^{\alpha +1}} \leq \frac{\sqrt{2}}{4} \lambda_i^{\beta+\frac12-\frac{\alpha}{2}}. \]
So by the dominated convergence theorem $\RomanIII_\eps\to 0$.
\end{proof}

We show the non-local to local convergence in two main steps.
In the first step we essentially use a Taylor expansion, $u(y)-u(x) \approx \langle D_{X^\gamma} u(x), y-x\rangle_{X^\gamma}$ to rewrite the non-local form using the derivative.

\begin{lemma}
\label{lem:proofs:nonlocal:nonloc_der}
Suppose Assumptions~\ref{ass:Setting:X}-\ref{ass:Setting:eta} are satisfied.
Let $\alpha \in (-1,+\infty)$ and $\beta\in \bbR$.
Assume $u\in\Ck{2}(X^\beta)$, $\sum_{i=1}^\infty \lambda_i^{\beta+\frac12-\frac{\alpha}{2}}<+\infty$, $\sum_{i=1}^\infty \lambda_i^{\alpha+1}<+\infty$, $\sum_{i=1}^\infty \lambda_i^{2(\beta+1)}<+\infty$ and $m_\mu\in X^\alpha\cap X^\beta$.
Then we have 
\[
\begin{aligned}
     & \lim_{\eps\to0} \Bigg|\frac{1}{ \varepsilon^2} \int_X\frac{\int_X \eta\left(\frac{\|y - x\|_{X^\alpha}}{\varepsilon}\right)| u(y) - u(x)|^2 \, \dd\mu(y)}{ \int_X \eta\left(\frac{\|y -x\|_{X^\alpha}}{\varepsilon}\right)\,\dd\mu(y) }\,\dd\mu(x) \\
     & \qquad \qquad -  \frac{1}{ \varepsilon^2} \int_X\frac{\int_X \eta\left(\frac{\|y - x\|_{X^\alpha}}{\varepsilon}\right)\langle D_{X^\beta}u(x), y - x\rangle_{X^\beta}^2 \, \dd\mu(y)}{ \int_X \eta\left(\frac{\|y -x\|_{X^\alpha}}{\varepsilon}\right)\,\dd\mu(y) }\,\dd\mu(x)  \Bigg| = 0.
\end{aligned}
\]
\end{lemma}

\begin{proof}
By Proposition~\ref{prop:Back:Gauss:mu_support} we have $\supp(\mu)\subset X^\beta$.
In the sequel, given $x,y\in \supp(\mu)$, we can assume $x,y\in X^\beta$.

For convenience, let us define 
\[ \zeta_\eps:= \frac{1}{\eps^2} \left|  \int_X\frac{\int_X \eta\left(\frac{\|y - x\|_{X^\alpha}}{\varepsilon}\right)\lp| u(y) - u(x)|^2 - \langle D_{X^\beta}u(x), y - x\rangle_{X^\beta}^2 \rp \, \dd\mu(y)}{ \int_X \eta\left(\frac{\|y -x\|_{X^\alpha}}{\varepsilon}\right) \, \dd\mu(y) } \, \dd\mu(x) \right|. \]
By Jensen's inequality we have 
\begin{equation}
\label{eq:proofs:nonlocal:tri_ineq}
\begin{aligned}
\zeta_\eps & \le \frac{1}{ \varepsilon^2} \int_X\frac{\int_X   \eta\left(\frac{\|y - x\|_{X^\alpha}}{\varepsilon}\right)\left| (| u(y) - u(x)|^2 - \langle D_{X^\beta}u(x), y - x\rangle_{X^\beta}^2 )\right|\,\dd\mu(y)}{ \int_X \eta\left(\frac{\|y -x\|_{X^\alpha}}{\varepsilon}\right)\,\dd\mu(y) }\,\dd\mu(x).  
\end{aligned}
\end{equation}
Since $u\in \Ck{2}(X^\beta)$,
\begin{equation} \label{eq:proofs:nonlocal:taylor_exp}
u(y) - u(x) = \langle D_{X^\beta} u(x),y-x\rangle_{X^\beta} + R_{x,y},
\end{equation}
where $R_{x,y} = \int_0^1 (1-\tau)\langle D_{X^\beta}^2u(x + \tau (y-x); y-x),y-x\rangle_{X^\beta} \, \dd\tau$. 
By the definition of $u\in\Ck{2}(X^\beta)$ there exists $M<+\infty$ such that $\|D_{X^\beta}^2 u(x;h)\|_{X^\beta} \leq M \|h\|_{X^\beta}$ for all $x,h\in X^\beta$.
Hence,
$$ |R_{x,y}| \le M \|y-x\|_{X^\beta}^2. $$

By~\eqref{eq:proofs:nonlocal:taylor_exp}, we have 
\begin{align*}
\la \la u(y) - u(x)\ra^2 - \langle D_{X^\beta} u(x), y-x \rangle_{X^\beta}^2  \ra &= \la 2\langle D_{X^\beta}u(x), y-x \rangle_{X^\beta} R_{x,y} + R_{x,y}^2 \ra \\
 & \le 2M \la \langle D_{X^\beta}u(x), y-x\rangle_{X^\beta}\ra \|y-x \|_{X^\beta}^2 + M^2 \|y-x \|_{X^\beta}^4.
\end{align*}
By inserting the inequality back into the right-hand of \eqref{eq:proofs:nonlocal:tri_ineq}, we acquire the following upper bound    
$$
\begin{aligned}
\zeta_\eps & \leq \frac{2M}{\eps^2} \int_X  \frac{\int_X \eta\left(\frac{\|y - x\|_{X^\alpha}}{\varepsilon}\right)|\langle D_{X^\beta}u(x), y - x\rangle_{X^\beta}|\| y-x\|_{X^\beta}^2\, \dd\mu(y)}{ \int_X \eta\left(\frac{\|y -x\|_{X^\alpha}}{\varepsilon}\right)\,\dd\mu(y) }\,\dd\mu(x) \\
 & \qquad + \frac{M^2}{ \eps^2} \int_X  \frac{\int_X \eta\left(\frac{\|y - x\|_{X^\alpha}}{\varepsilon}\right)\| y-x\|_{X^\beta}^4\, \dd\mu(y)}{ \int_X \eta\left(\frac{\|y -x\|_{X^\alpha}}{\varepsilon}\right)\,\dd\mu(y) }\,\dd\mu(x) \\
 & \leq \underbrace{\frac{\delta M}{ \eps^2} \int_X  \frac{\int_X \eta\left(\frac{\|y - x\|_{X^\alpha}}{\varepsilon}\right)\langle D_{X^\beta}u(x), y - x\rangle_{X^\beta}^2\, \dd\mu(y)}{ \int_X \eta\left(\frac{\|y -x\|_{X^\alpha}}{\varepsilon}\right)\,\dd\mu(y) }\,\dd\mu(x)}_{\to \frac{\delta M}{2}\int_X  \left \| D_{X^\beta} u(x)\right\|_{X^{2\beta-\alpha}}^2 \, \dd \mu(x)} \\
 & \qquad + \l M+\frac{1}{\delta}\r \underbrace{\frac{M}{\eps^2} \int_X  \frac{\int_X \eta\left(\frac{\|y - x\|_{X^\alpha}}{\varepsilon}\right)\| y-x\|_{X^\beta}^4\, \dd\mu(y)}{ \int_X \eta\left(\frac{\|y -x\|_{X^\alpha}}{\varepsilon}\right)\,\dd\mu(y) }\,\dd\mu(x)}_{\to 0} \\
 & \to \frac{\delta M}{2}\int_X  \left \| D_{X^\beta} u(x)\right\|_{X^{2\beta-\alpha}}^2 \, \dd \mu(x)
\end{aligned}
$$
as $\eps\to 0$, by Lemma~\ref{lem:proofs:nonlocal:4thmoment} and Proposition~\ref{prop:proofs:nonlocal:final} for any $\delta>0$ using $|ab|\leq \frac{\delta a^2}{2} + \frac{b^2}{2\delta}$. 
Since $\lim_{\eps\to0} \zeta_\eps \leq \frac{\delta M}{2}\int_X  \left \| D_{X^{\beta}} u(x)\right\|_{X^{2\beta-\alpha}}^2 \, \dd \mu(x)$ for all $\delta>0$ then we take $\delta\to 0$ to conclude the result.
\end{proof}

The second (and final) step of the non-local to local convergence is to show that we can remove the inner product and use only the norm of the gradient.

\begin{proposition}
\label{prop:proofs:nonlocal:final}
Suppose Assumptions~\ref{ass:Setting:X}-\ref{ass:Setting:eta} are satisfied.
Let $\alpha\in (-1,+\infty)$ and $\beta\in\bbR$.
Assume $u\in\Ck{1}(X^\beta)$, $\sum_{i=1}^\infty \lambda_i^{\beta+\frac12-\frac{\alpha}{2}}<+\infty$, $\sum_{i=1}^\infty\lambda_i^{\alpha+1}<+\infty$, $\sum_{i=1}^\infty\lambda_i^{2(\beta+1)}<+\infty$ and $m_\mu\in X^\alpha\cap X^\beta$.
Then we have 
\[ 
\lim_{\eps \to 0} \frac{1}{\eps^2} \int_X \frac{\int_X \eta\l\frac{\|y - x\|_{X^\alpha}}{\eps}\r\langle D_{X^\beta}u(x), y - x\rangle^2_{X^\beta} \, \dd\mu(y)}{ \int_X \eta\l\frac{\|y -x\|_{X^\alpha}}{\eps}\r \, \dd\mu(y) }\,\dd\mu(x) =  \frac12 \int_X  \| D_{X^\beta}u(x)\|^2_{X^{2\beta-\alpha}} \, \dd \mu(x).
\]    
\end{proposition}

\begin{proof}
Take $f(z) = \langle A(x), z\rangle_{X^\beta}^2$ where $A(x) = D_{X^\beta}u(x)$ in Lemma~\ref{lem:Back:Tech:Conj} to infer, for any $x\in X^\alpha$, 
\[ \frac{\int_X \eta\l\frac{\|y-x\|_{X^\alpha}}{\eps}\r \langle D_{X^\beta} u(x), y-x\rangle_{X^\beta}^2 \, \dd \mu(y)}{\int_X \eta\l \frac{\|x-y\|_{X^\alpha}}{\eps}\r \, \dd \mu(y)} = \int_X \langle A(x), y-x\rangle_{X^\beta}^2 \, \dd \mu_{x,\eps}(y) \]
with $\mu_{x,\eps} = \mathcal{N}(m_{x,\eps},C_{x,\varepsilon})$ where $m_{x,\eps}$ and $C_{x,\eps}$ are given by~\eqref{eq:Back:Tech:mxeps} and~\eqref{eq:Back:Tech:Cxeps}.
Manipulating the RHS we can continue,
\begin{align*}
& \frac{\int_X \eta\l\frac{\|y-x\|_{X^\alpha}}{\eps}\r \langle D_{X^\beta} u(x), y-x\rangle_{X^\beta}^2 \, \dd \mu(y)}{\int_X \eta\l \frac{\|x-y\|_{X^\alpha}}{\eps}\r \, \dd \mu(y)} \\
& \qquad \qquad = \int_X \l \sum_{i=1}^\infty \lambda_i^\beta A(x)_i (y_i-x_i)\r ^2 \, \dd \mu_{x,\eps}(y) \\
& \qquad \qquad = \sum_{i,j=1}^\infty \int_X \lambda_i^\beta \lambda_j^\beta A(x)_i A(x)_j (y_i-x_i)(y_j-x_j) \, \dd \mu_{x,\eps}(y) \\
& \qquad \qquad = \sum_{i\neq j} \lambda_i^\beta \lambda_j^\beta A(x)_i A(x)_j \int_{\bbR^2} \l \sqrt{\bar{\lambda}_i}\xi_i + \bar{m}_i - x_i\r \l \sqrt{\bar{\lambda}_j} \xi_j + \bar{m}_j - x_j\r \rho(\xi_i)\rho(\xi_j) \, \dd \xi_i \, \dd \xi_j \\
& \qquad \qquad \qquad \qquad + \sum_{i=1}^\infty \lambda_i^{2\beta} A(x)_i^2 \int_{\bbR} \l \sqrt{\bar{\lambda}_i} \xi_i + \bar{m}_i - x_i\r^2 \rho(\xi_i) \, \dd \xi_i \\
& \qquad \qquad = \sum_{i\neq j} \lambda_i^\beta \lambda_j^\beta A(x)_i A(x)_j (\bar{m}_i-x_i)(\bar{m}_j-x_j) + \sum_{i=1}^\infty \lambda_i^{2\beta} A(x)_i^2 \l \bar{\lambda}_i + (\bar{m}_i - x_i)^2 \r \\
& \qquad \qquad = \sum_{i=1}^\infty \lambda_i^{2\beta} \bar{\lambda}_i A(x)_i^2 + \l \sum_{i=1}^\infty \lambda_i^\beta A(x)_i (\bar{m}_i-x_i)\r^2 \\
& \qquad \qquad = \eps^2 \sum_{i=1}^\infty \frac{\lambda_i^{2\beta+1}}{\eps^2+2\lambda_i^{\alpha+1}} A(x)_i^2 + \eps^4 \l \sum_{i=1}^\infty \frac{\lambda_i^\beta(m_i-x_i)}{\eps^2+2\lambda_i^{\alpha+1}} A(x)_i\r^2
\end{align*}
where $\bar{m}_i = \frac{\eps^2m_i + 2\lambda_i x_i}{\eps^2+2\lambda_i}$, $m_i = \langle m_\mu,e_i\rangle_X$, $\bar{\lambda}_i=\frac{\eps^2\lambda_i}{\eps^2+2\lambda_i^{\alpha+1}}$ and $\rho$ is the standard Gaussian density on $\bbR$.
So, (following the argument in the proof of Lemma~\ref{lem:proofs:nonlocal:4thmoment} we restrict the outer integral to $x\in X^\alpha \cap X^\beta$)
\begin{align*}
& \frac{1}{\eps^2} \int_X \frac{\int_X \eta\l\frac{\|y-x\|_{X^\alpha}}{\eps}\r \langle D_{X^\beta} u(x), y-x\rangle_{X^\beta}^2 \, \dd \mu(y)}{\int_X \eta\l \frac{\|x-y\|_{X^\alpha}}{\eps}\r \, \dd \mu(y)} \, \dd \mu(x) \\
& \qquad \qquad = \underbrace{\int_X \sum_{i=1}^\infty \frac{\lambda_i^{2\beta+1}A(x)_i^2}{\eps^2 + 2\lambda_i^{\alpha+1}}  \, \dd \mu(x)}_{=:\RomanV_\eps} + \underbrace{\eps^2\int_X \l \sum_{i=1}^\infty \frac{\lambda_i^\beta (m_i-x_i) A(x)_i}{\eps^2 + 2\lambda_i^{\alpha+1}} \r^2 \, \dd \mu(x)}_{\RomanVI_\eps}.
\end{align*}
We will consider the two terms, $\RomanV_\eps$ and $\RomanVI_\eps$ separately.

For the first term, $\RomanV_\eps$, we have a pointwise bound: 
\[ \frac{\lambda_i^{2\beta+1}}{\eps^2 + 2\lambda_i^{\alpha+1}} A(x)_i^2 \leq \frac{\lambda_i^{2\beta-\alpha} A(x)_i^2}{2} \text{ and } \sum_{i=1}^\infty \lambda_i^{2\beta-\alpha} A(x)_i^2 = \|D_{X^\beta} u(x) \|_{X^{2\beta-\alpha}}^2\in \Lp{1}(\mu). \]
So by the dominated convergence theorem,
\[ \lim_{\eps\to 0} \RomanV_\eps = \frac12 \int_X \sum_{i=1}^\infty \lambda_i^{2\beta-\alpha} A(x)_i^2 \, \dd \mu(x) = \frac12 \int_X \| D_{X^\beta} u(x)\|_{X^{2\beta-\alpha}}^2 \, \dd \mu(x). \]

For the second term, $\RomanVI_\eps$, we have,
\begin{align*}
\RomanVI_\eps & \leq \eps^2 \int_X \| D_{X^\beta}u(x)\|_{X^\beta} \sqrt{\sum_{i=1}^\infty \frac{\lambda_i^\beta (m_i-x_i)^2}{(\eps^2+2\lambda_i^{\alpha+1})^2}} \, \dd \mu(x) \\
 & \leq \eps^2 \sup_{x\in X^\beta} \| D_{X^\beta} u(x) \|_{X^\beta} \sqrt{\sum_{i=1}^\infty \int_{\bbR} \frac{\lambda_i^{1+\beta} \xi_i^2}{(\eps^2+2\lambda_i^{\alpha+1})^2} \rho(\xi_i) \, \dd \xi_i} \\
 & = \eps^2 \sup_{x\in X^\beta} \| D_{X^\beta} u(x) \|_{X^\beta} \sqrt{\sum_{i=1}^\infty \frac{\lambda_i^{1+\beta}}{(\eps^2+2\lambda_i^{\alpha+1})^2}}.
\end{align*}
Defining $\RomanIV_\eps$ as in the proof of Lemma~\ref{lem:proofs:nonlocal:4thmoment} and using that $\RomanIV_\eps\to 0$ we conclude
\[ \RomanVI_\eps \leq \sqrt{\eps} \sup_{x\in X^\beta} \| D_{X^\beta} u(x) \|_{X^\beta} \sqrt{\RomanIV_\eps} \to 0. \qedhere \]
\end{proof}    

Combining Lemma~\ref{lem:proofs:nonlocal:nonloc_der} and Proposition~\ref{prop:proofs:nonlocal:final} we can conclude the non-local continuum to local convergence.

\begin{theorem}
\label{thm:proofs:nonlocal:nonlocal}
Suppose Assumptions~\ref{ass:Setting:X}-\ref{ass:Setting:eta} are satisfied.
Let $\alpha\in (-1,+\infty)$ and $\beta\in \bbR$.
Assume $u\in\Ck{2}(X^\beta)$, $\sum_{i=1}^\infty \lambda_i^{\beta+\frac12-\frac{\alpha}{2}}<+\infty$, $\sum_{i=1}^\infty\lambda_i^{\alpha+1}<+\infty$, $\sum_{i=1}^\infty\lambda_i^{2(\beta+1)}<+\infty$ and $m_\mu\in X^\alpha\cap X^\beta$.
Then we have 
\[ \lim_{\eps \to 0}  \frac{1}{ \eps^2} \int_X\frac{\int_X \eta\l\frac{\|y - x\|_{X^\alpha}}{\eps}\r| u(y) - u(x)|^2 \, \dd\mu(y)}{ \int_X \eta\l\frac{\|y -x\|_{X^\alpha}}{\eps}\r \, \dd\mu(y) }\,\dd\mu(x)  =  \frac12 \int_X \| D_{X^{\beta}} u(x)\|_{X^{2\beta-\alpha}}^2 \, \dd \mu(x). \]
\end{theorem}

\section{Numerical Experiments}
\label{sec:num_exp}

In the numerical experiments section, we compare the proposed infinite-dimensional framework with the usual finite-dimensional one.
We stress that  these settings differ in two distinct ways. The first one is the use of the normalized Laplacian --- which makes sense in  infinite dimensions  --- versus the unnormalized one which is commonly used in the finite-dimensional setting.  The second one is the geometry motivated by infinite-dimensional considerations, such as the use of negative-order Sobolev norms for ``rough'' objects such as white noise.

Throughout this section, we assume that all  data $x_i$ are inherently infinite-dimensional, but we only have access to their discretized versions, as it is typical in applications. 
The main goal of our investigations is to study the performance of different methods as a function of the dimension of the discretized space, the common wisdom being that methods that have a well-defined infinite-dimensional formulation perform better in finite but high dimensions than inherently finite-dimensional ones; this is indeed what our results seem to confirm.

When the Euclidean norm is chosen in the weight function, our method naturally recovers the (finite-dimensional) normalized Graph Laplacian.
Alternatively, one can estimate the underlying infinite-dimensional data using a basis expansion and then use an appropriate norm in the weight function. This enables the incorporation of additional structural information about the data into the semi-supervised learning task, improving the performance. 
The advantage is particularly evident when the data lie in spaces less regular than $\Lp{2}$.

This section is organized as follows. In Section \ref{sec:num_method}, we describe the feature vectors and labels used in our experiments and explain how we solve the semi-supervised learning problem using the proposed normalized Laplacian for infinite-dimensional data and the classical unnormalized Laplacian for Euclidean data, under the assumption that only discretized versions of the infinite-dimensional data are available. 
In Section \ref{sec:white_noise}, we study data corrupted by additive white noise. Since such data lie in spaces rougher than $\Lp{2}$, choosing an appropriate (rougher) norm is crucial. 
In Section \ref{sec:Brownian_bridge}, we consider semi-supervised learning for the Brownian bridge, focusing on three tasks: labeling low-frequency data, labeling high-frequency data, and predicting the maximum value.  
For these labeling tasks, if additional information is available about which modes contribute most to the label, this knowledge can be incorporated into the norm via a reweighting, which improves the performance compared to the standard Euclidean approach.

\subsection{Methods}
\label{sec:num_method}

Let $x_j \overset{\text{i.i.d.}}{\sim} \mu$, $j=1,...,n$, denote feature vectors drawn in an independent and identically distributed way from the probability measure $\mu\in \cP(X)$ over a separable Hilbert space $X$. 
In our numerical experiments, we stay with the original norm of the Hilbert space $X$, that is we choose $\alpha =0$ and $\beta = 0$ in Theorem~\ref{thm:Setting:PointCons}.

Throughout this section, we assume that an oracle has access to the ``true'' infinite-dimensional feature vectors $x_j$, whereas the practical learner is only given a finite-dimensional vector $\bar{x}_j \in \mathbb{R}^{N}$, usually composed of  evaluations on a finite grid
\begin{equation}\label{eq:finite-data}
    \bar{x}_j = \begin{bmatrix}
        \bar{x}_j(1), \ldots, \bar{x}_j(N)
    \end{bmatrix}^{\top} := \begin{bmatrix}
    x_j(t_1), \ldots,x_j(t_{N})
\end{bmatrix}^{\top} \in \mathbb{R}^{N},
\end{equation}
where $t_1, \ldots, t_{N}$ are the observation points. 

Let  $y_1^*, \ldots, y_{\ell}^*$ be given labels with $\ell \ll n$ and denote by $y_{\ell+1}, \ldots, y_n$  the labels predicted by the learner. Let $y_{\ell+1}^*, \ldots, y_n^*$ be the corresponding true labels. We measure the labeling error rate by
$$
1-\frac{\sum_{i=\ell+1}^n \mathbb{I}_{y_i^*}\left(y_i\right)}{n-\ell},
$$
where $\mathbb{I}_a(b)$ is the indicator function defined by
$$
\mathbb{I}_a(b):= \begin{cases}1, & b=a, \\ 0, & b \neq a .\end{cases}
$$

Below, we consider two different semi-supervised learning settings.
In all settings the connection radius $\eps$ in Eq. \eqref{eq:Setting:weight_function} is selected as the smallest value for which the graph remains connected.

\subsubsection{Laplace Learning in Finite Dimensions: the Unnormalized Laplacian}
\label{sec:laplace_learning_euclidean_spaces}
In the practical setting, we aim to label the discretized data. Given $\{
\bar{x}_i, y^*_i\}_{i=1}^{\ell}$, our goal is to find a labeling function defined on the vertices of the graph,
$$u:\{\bar{x}_1,\ldots,\bar{x}_\ell \} \rightarrow \mathbb{R},$$ that minimizes the objective function 
$$\sum_{i=1}^n \sum_{j=1}^n W_{i j}\left|u\left(\bar{x}_i\right)-u\left(\bar{x}_j\right)\right|^2,$$ 
where the weights are given by $W_{i j}=\exp \left(-\frac{\left\|\bar{x}_i-\bar{x}_j\right\|^2}{2 \varepsilon^2}\right)$ and $\| \cdot\|$ denotes the Euclidean norm.
Let $\mathbf{u}:=\left[\begin{array}{c}u\left(x_1\right) \\ \vdots \\ u\left(x_n\right)\end{array}\right]$ and  denote $\mathbf{u} = \begin{bmatrix}
    \mathbf{u}_\ell\\
    \mathbf{u}_u
\end{bmatrix}$ with
$
\mathbf{u}_\ell:=\left[\begin{array}{c}u\left(x_1\right) \\ \vdots \\ u\left(x_\ell\right)\end{array}\right]
$ 
and
$
\mathbf{u}_u:=\left[\begin{array}{c}u\left(x_{\ell+1}\right) \\ \vdots \\ u\left(x_n\right)\end{array}\right].
$
Using matrix operations, the objective function can be written equivalently as $\mathbf{u}^{\top}L\mathbf{u}$, where $L:=D-W$ and $D_{i i}=\sum_{j=1}^n W_{i j}$. Partitioning $L$ according to the labeled and unlabeled indices,
$$
L=\left[\begin{array}{ll}
L_{u u} & L_{u \ell} \\
L_{\ell u} & L_{\ell \ell}
\end{array}\right],
$$
und using the first-order optimality condition for the quadratic objective function $\nabla\left(\mathbf{u}^{\top} L\mathbf{u}\right)=0$, we conclude that the   optimal labeling on the unlabeled feature vectors is given by $\mathbf{u}_u^*=-L_{u u}^{-1}L_{u \ell} \mathbf{u}_{\ell}$.

\subsubsection{Laplace Learning in Infinite Dimensions: the Normalized Laplacian}
\label{sec:laplace_learning_func_spaces}
Given $\{
x_i, y^*_i\}_{i=1}^{\ell}$ for $x_i$ in the Hilbert space, the normalized Laplacian aims to find a labeling function defined on the vertices of the graph,
$$u:\{x_1,\ldots,x_\ell \} \rightarrow \mathbb{R},$$ that minimizes the objective function
$$\sum_{i=1}^n \frac{\sum_{j=1}^n W_{i j}\left|u\left(x_i\right)-u\left(x_j\right)\right|^2}{\sum_{j=1}^n W_{i j}},$$ 
where the weights are given by $W_{i j}=\exp \left(-\frac{\left\|x_i-x_j\right\|_{X^{\alpha}}^2}{2 \varepsilon^2}\right)$ and $\| \cdot\|_{X^{\alpha}}$ is the fractional norm on $X$ defined in~\eqref{eq:frac-norm}.
Let $\mathbf{u}$, $\mathbf{u}_\ell$ 
and $\mathbf{u}_u$ be as before.
Using matrix operations, the objective function can be written equivalently as $\mathbf{u}^{\top}L\mathbf{u}$, where $L:=I-2 P+D$ with $P_{i j}:=\frac{W_{i j}}{d_i}, d_i=\sum_{j=1}^n W_{i j}$ and $D=\left[\begin{array}{lll}\sum_i P_{i 1} & & \\ & \ddots & \\ & & \sum_i P_{i n}\end{array}\right]$. Upon partitioning $L$ according to the labeled and unlabeled indices as above, the first-order optimality condition for the quadratic objective function $\nabla\left(\mathbf{u}^{\top} L\mathbf{u}\right)=0$ yields the optimal labeling on the unlabeled feature vectors 
\begin{equation}
    \label{eq:sol_oracle}
    \mathbf{u}_u^*=-\left(L_{u u}^{\top}+L_{u u}\right)^{-1}\left(L_{\ell u}^{\top}+L_{u \ell}\right) \mathbf{u}_{\ell}.\end{equation}

However, in the practical setting, we can only get access to the discretized data. Hence, given $\{
\bar{x}_i, y^*_i\}_{i=1}^{\ell}$, we consider the following two approaches for the Laplace learning in infinite-dimensional Hilbert spaces: (i) approximating the norm of the underlying Hilbert space and applying it to finite-dimensional data, and (ii) (approximately) reconstructing the infinite-dimensional data from the discrete data. Essentially, (i) is a finite difference method and (ii) is a spectral method.

For the first method, we approximate the Hilbert-space norm with the discretized data. In particular, for the $\Lp{2}$ norm with empirical measure on the grid we have 
\begin{equation}\label{eq:L2_discretized_data}
    \|x-y \|_{\Lp{2}}^2 = \int |x(t)-y(t)|^2 \, \dd\mu(t) \approx \frac{1}{N_{grid}}\sum_{i=1}^{N_{grid}}|\bar{x}(i)-\bar{y}(i)|^2.
\end{equation}

The second approach consists in recovering the expansion of the infinite-dimensional data in some basis. Let $\{e_k\}$ be an orthonormal basis of the Hilbert space $X$, so that  $x = \sum_{k = 1}^{
\infty
}\langle x, e_k\rangle e_k$ for any $x \in X$. In the $\Lp{2}$ case, given the finite-dimensional data $\bar x_j$ from~\eqref{eq:finite-data}, we aim to estimate the coefficients $a_{jk} := \langle x_j, e_k\rangle$ by
\begin{equation}
\label{eq:coef_learning_no_weight}
    a_{jk} = \int x_j(t)e_k(t) \, \dd\mu(t) \approx 1/N_{grid} \sum_{i=1}^{N_{grid}} x_j(t_i) e_k(t_i)  =: \bar a_{jk}.
\end{equation}
We truncate the expansion of $x_j$ after $m=N/2$ modes in accordance with the Nyquist–Shannon criterion, that is $x_j \approx \sum_{k=1}^m \bar a_{jk} e_k$. 
By Parceval's theorem, we have that
\begin{equation}
    \|x\|_X^2 = \sum_{k=1}^\infty \left|\langle x, e_k\rangle\right|^2, \quad x \in X.
\end{equation}
If we know that certain frequencies are more important than others, we can introduce a weighted (equivalent) norm using a vector $c \in \ell^\infty$, $c_k \geq \delta>0$ for all $i$,
\begin{equation}\label{eq:weighted-norm}
    \|x\|_{X,c}^2 = \sum_{k=1}^\infty c_k\left|\langle x, e_k\rangle\right|^2, \quad x \in X.
\end{equation}

\subsection{Shifted White Noise}
\label{sec:white_noise}
\subsubsection{Features and Labels}
In this section, we consider an example where the features aren't functions but live in a rougher space -- white noise. 
In this setting we will see that a finite-dimensional graph Laplacian may lead to an inadequate labeling, especially at high resolutions,  whereas by selecting an appropriate norm (negative Sobolev norm) we can solve the semi-supervised learning task well.

Let $X = \Hk{-s}([0,1])$ with $s>0$ be a negative-order Sobolev space. We choose the orthonormal basis $\{e_k\}$ of $L^{2}([0,1])$ given by
\begin{equation}\label{eq:ei-fi-cos}
    \{e_k\}= \{1\} \cup \{\sqrt{2}\cos(\pi k \cdot)\}_{k = 1, 2, \ldots}  \quad \text{and define} \quad f_k = \left(1+(k \pi)^2\right)^{s / 2} e_k.
\end{equation}
Assuming periodic boundary conditions, $\{ f_k\}_{k=0,1,2,\ldots}$ is an orthonormal basis of $\Hk{-s}$, and the norm is given by
\begin{equation}
    \label{eq:frac_sobolev_norm_cos}
    \|x\|_{\Hk{-s}}^2 = \sum_{k=0}^{\infty}\left(1+(k \pi)^2\right)^{-s}\left|\langle x, e_k\rangle_{\Lp{2}}\right|^2=\sum_{k=0}^{\infty}\left(1+(k \pi)^2\right)^{-2s}\left|\langle x, f_k\rangle_{\Lp{2}}\right|^2, \quad x \in \Hk{-s}([0,1]).
\end{equation}

The covariance of white noise is the identity operator on $\Lp{2}([0,1])$, and any orthonormal basis of $\Lp{2}([0,1])$ is an eigenbasis (with eigenvalues $1$). We will use the basis $\{e_k\}$ from~\eqref{eq:ei-fi-cos}. With $\xi_k \iid \cN(0,1)$, white noise is defined via the Karhunen-Lo\'eve expansion as follows \cite{carrizo2025generalized}
\begin{equation*}
    W = \sum_{k} \xi_k e_k = \sum_{k} \left(1+(k \pi)^2\right)^{-s/2} \xi_k f_k.
\end{equation*}
On $\Hk{-s}([0,1]))$, the eigenvalues of the covariance  are given by $\lambda_k = \left(1+(k \pi)^2\right)^{-s}$ and satisfy the conditions of Theorem~\ref{thm:Setting:PointCons} with $\alpha=\beta=0$ as long as $s>1$. We choose $s=1.01$ in our experiments.

We generate each feature vector from a constant signal corrupted by white noise and evaluated on uniform finite grids. For each sample $j$, the feature vector is defined as $x_j = \mu_j + w_j$, where $\mu_j \in\{+1,-1\}$ is a constant function and $w_j$ denotes a realization of Gaussian white noise on $[0,1]$. 

For the time-discretized data, we approximate $w_j$ on a uniform grid $0=t_1<t_2<\cdots<t_m=1$
by sampling independent Gaussian values with variance $1 / \Delta t=m$.
Equivalently, the discretized white noise is represented as
$$
w_j\left(t_i\right) \approx \sqrt{m} z_{j, i}, \quad z_{j, i} \sim \mathcal{N}(0,1),
$$
which corresponds to the classical approximation of white noise using a piecewise-constant $\Lp{2}$ orthonormal basis on $[0,1]$. Thus each discretized feature vector is of the form
$$
\bar{x}_j=\left[x_j\left(t_1\right), \ldots, x_j\left(t_m\right)\right]^{\top}=\mu_j \mathbf{1}_m+\sqrt{m} z_j, \quad z_j \sim \mathcal{N}\left(0, I_m\right).
$$

We label $x_j$ to be $+1$ if $\mu_j = 1$, otherwise $-1$, i.e. 
	$$
y_j = \begin{cases} +1 & \text { if }  \mu_j = 1 \\ 
-1 & \text { otherwise. }\end{cases}
$$

\subsubsection{Results}
We generate 200 feature vectors and label the first 20. We compare the performance of the normalized Laplacian using the $\Lp{2}$ or $\Hk{-s}$ with $s=1.01$ geometries with the unnormalized Laplacian in Euclidean space. 
For the infinite-dimensional framework using $\Lp{2}$ and $\Hk{-s}$ from Section \ref{sec:laplace_learning_func_spaces}, we reconstruct infinite-dimensional data on $X = \Lp{2}$ and $X = \Hk{-s}$   using the (truncated) bases $\{e_k\}_{k=0,...,m}$ and $\{f_k\}_{k=0,...,m}$, respectively, as defined in~\eqref{eq:ei-fi-cos}. The coefficients are estimated via \eqref{eq:coef_learning_no_weight}. Hence each  data point takes the form $\bar a_0 + \sum_{k=1}^{m} \bar a_k \cos(k\pi t)$.

Figure \ref{fig:whitenoise_grid_panel} shows the labeling of four realizations obtained using the unnormalized Laplacian with the Euclidean norm and the normalized Laplacian with the $\Hk{-s}$ norm, as well as  true labels. We show results for seven different resolutions (10, 20, 40, ..., 640 points). We see that whilst at low resolutions both methods perform similarly, the performance of the unnormalized Laplacian with the Euclidean geometry deteriorates at high resolutions while that of the normalized Laplacian with the $\Hk{-s}$ geometry stays consistently good at all resolutions.

\newcommand{\imgw}{0.28\textwidth}

\begin{figure}[htp!]
  \centering
  \setlength{\tabcolsep}{4pt}    
  \renewcommand{\arraystretch}{1}

  \begin{tabular}{r c c c}
    \textbf{$N=10$} &
      \includegraphics[width=\imgw]{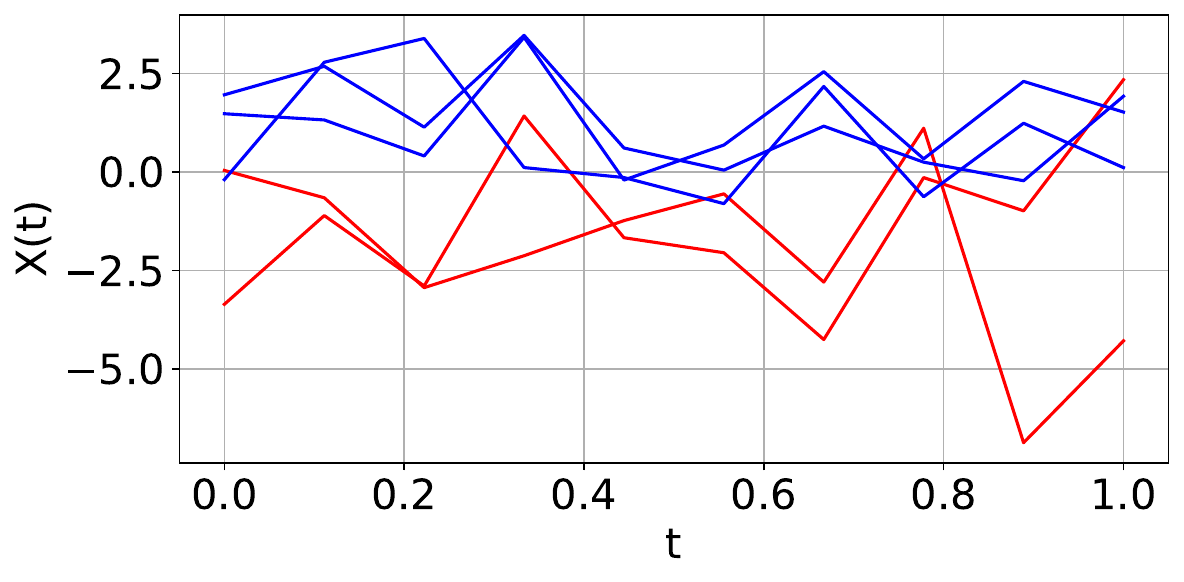} &
      \includegraphics[width=\imgw]{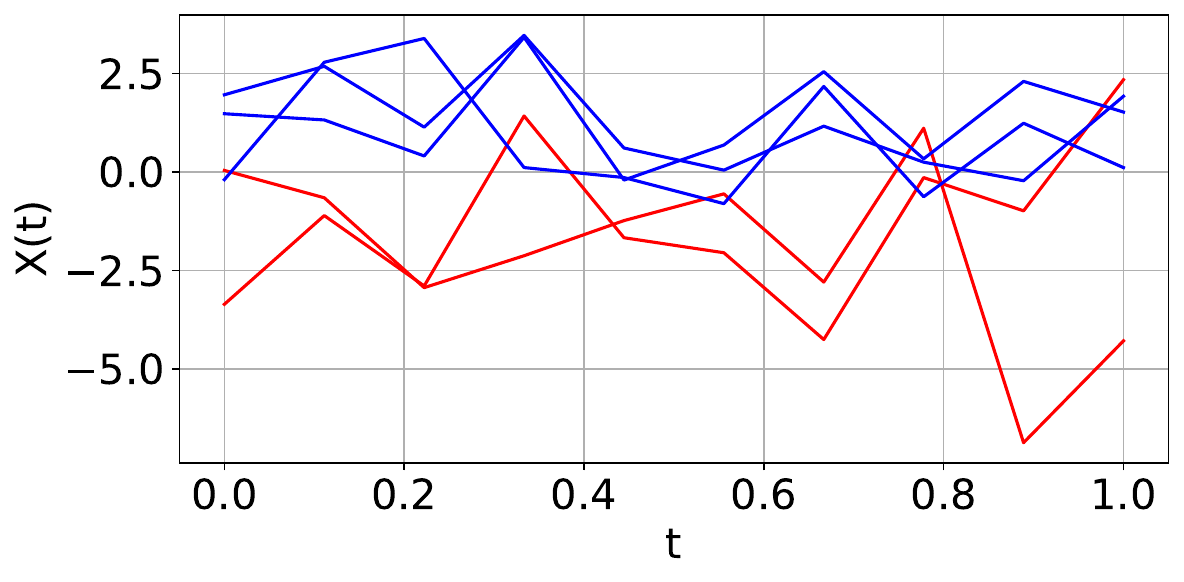} &
      \includegraphics[width=\imgw]{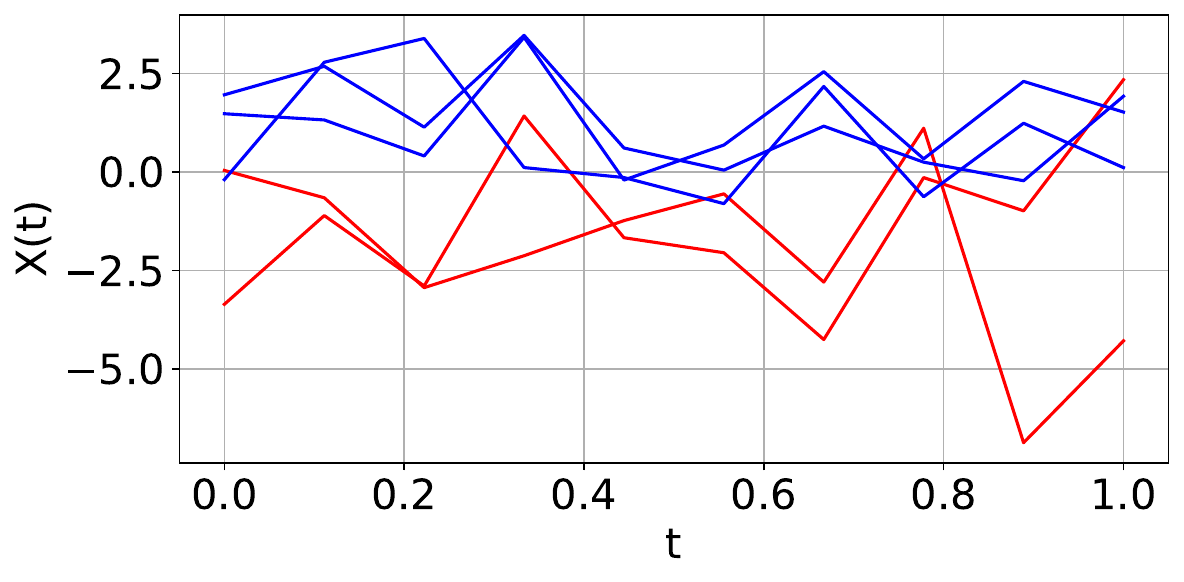} \\
    \textbf{$N=20$} &
      \includegraphics[width=\imgw]{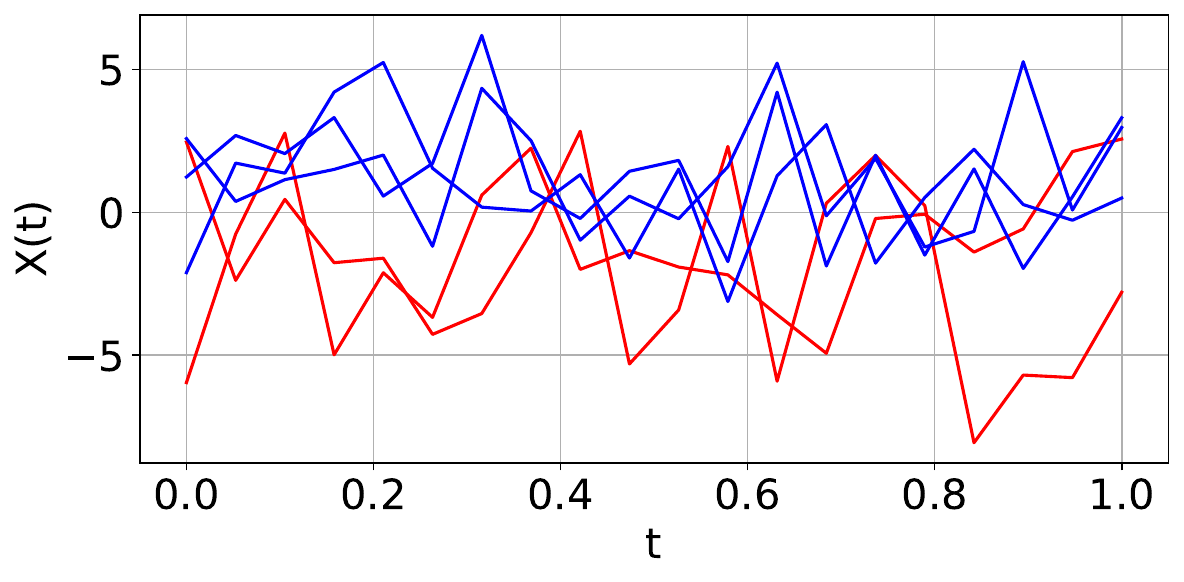} &
      \includegraphics[width=\imgw]{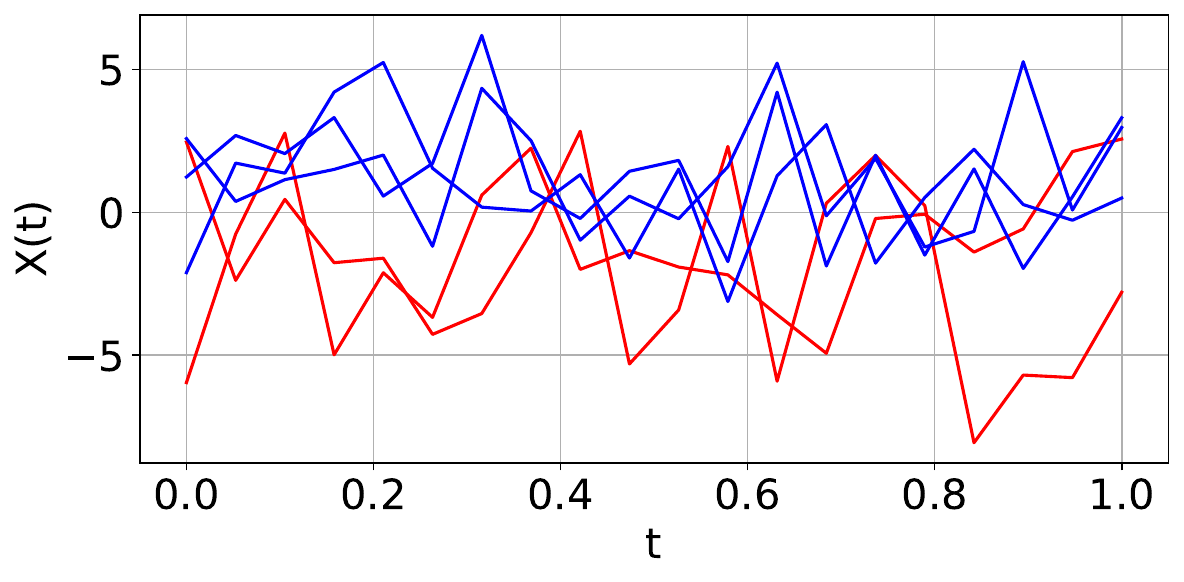} &
      \includegraphics[width=\imgw]{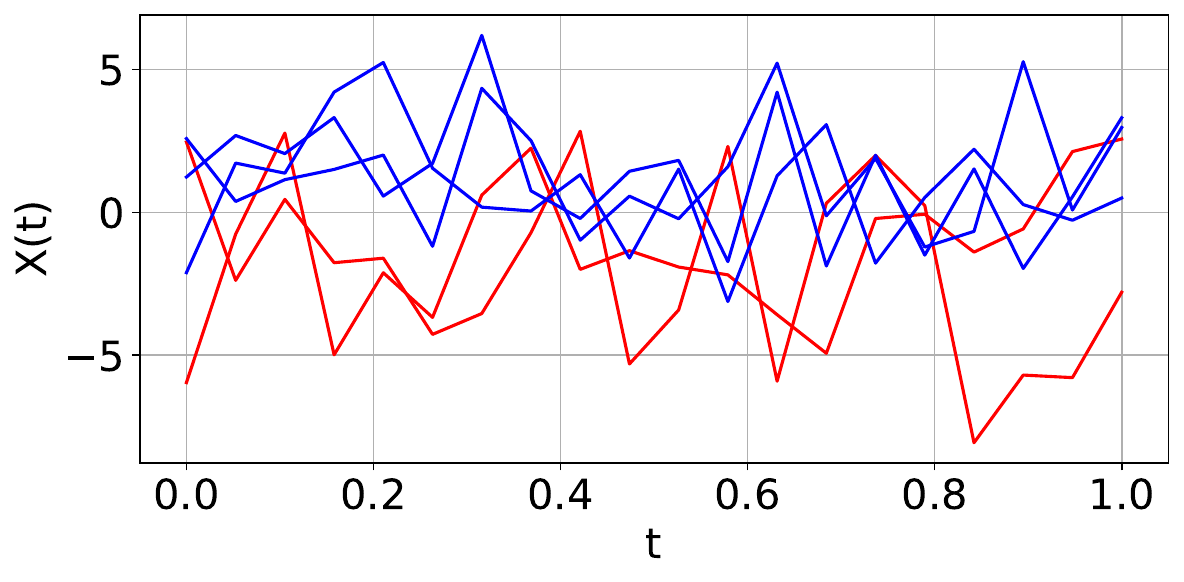} \\
    \textbf{$N=40$} &
      \includegraphics[width=\imgw]{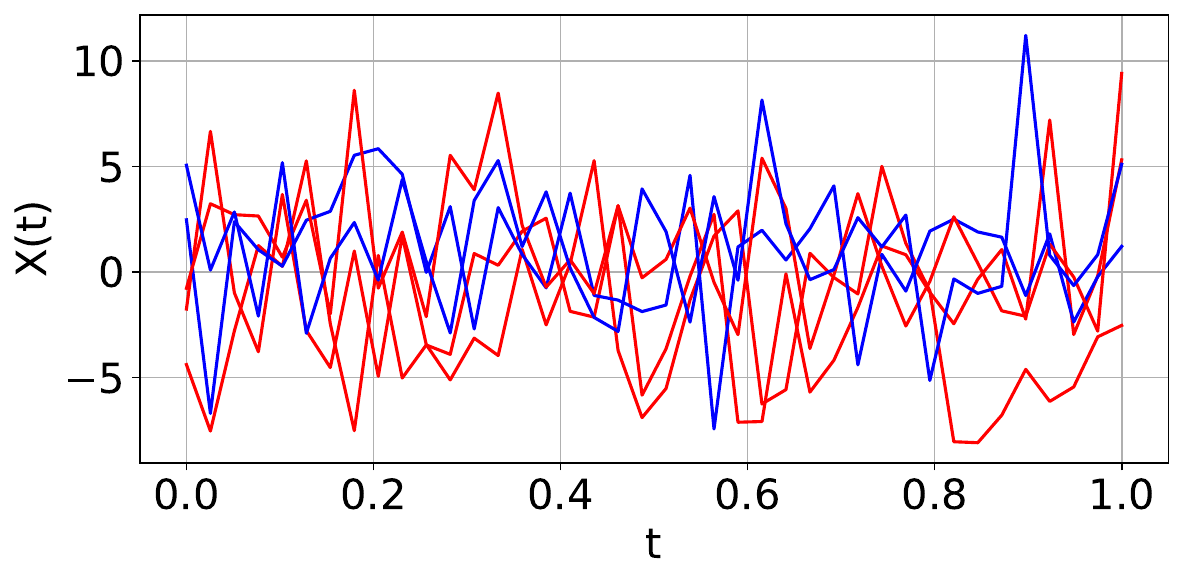} &
      \includegraphics[width=\imgw]{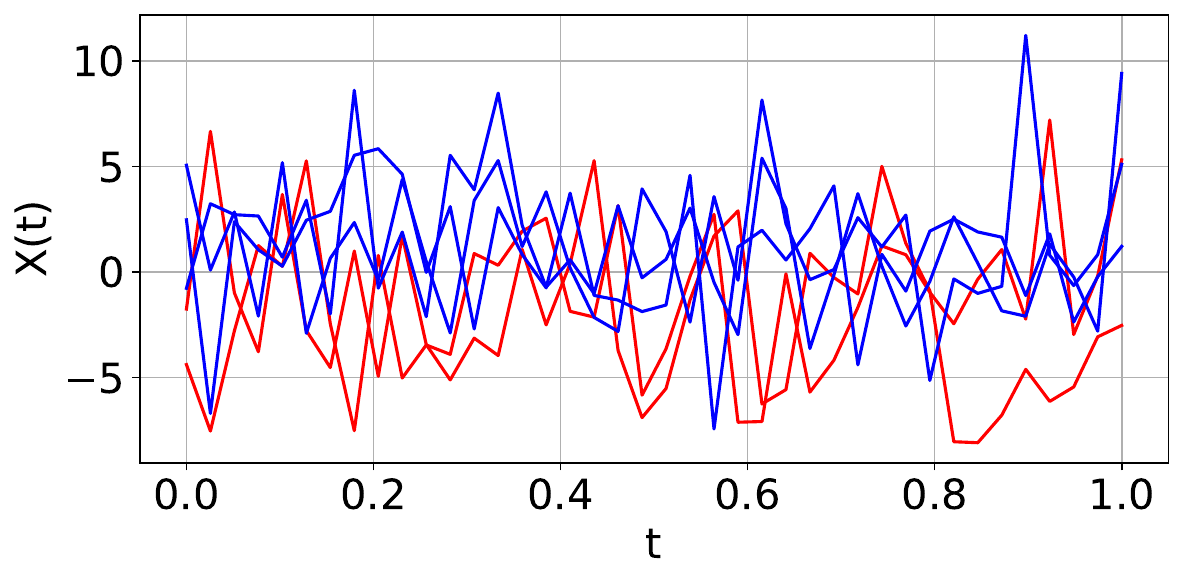} &
      \includegraphics[width=\imgw]{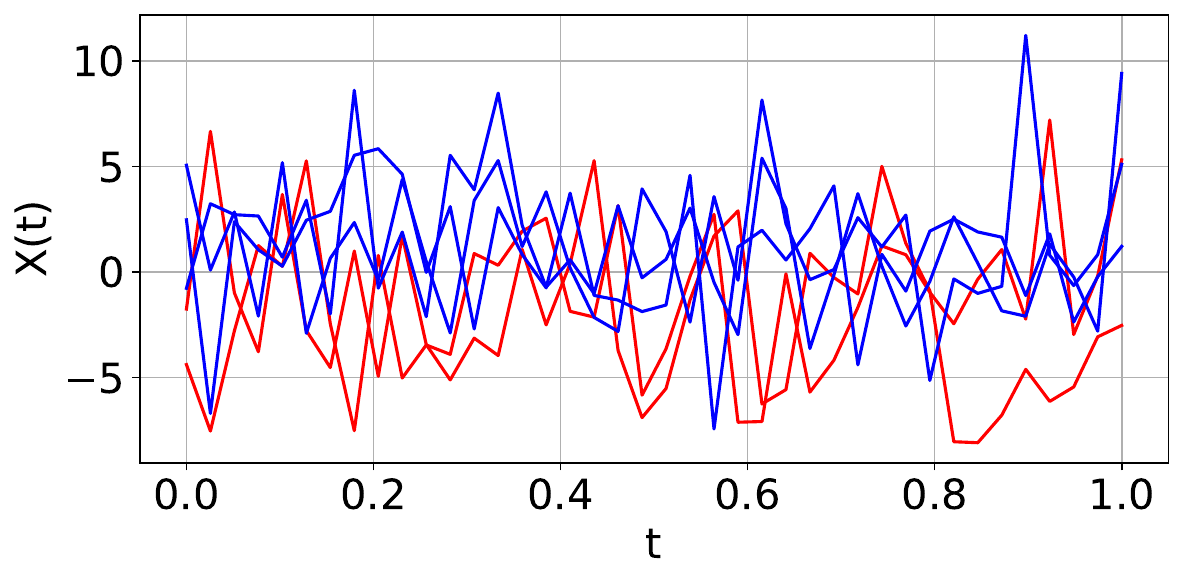} \\
    \textbf{$N=80$} &
      \includegraphics[width=\imgw]{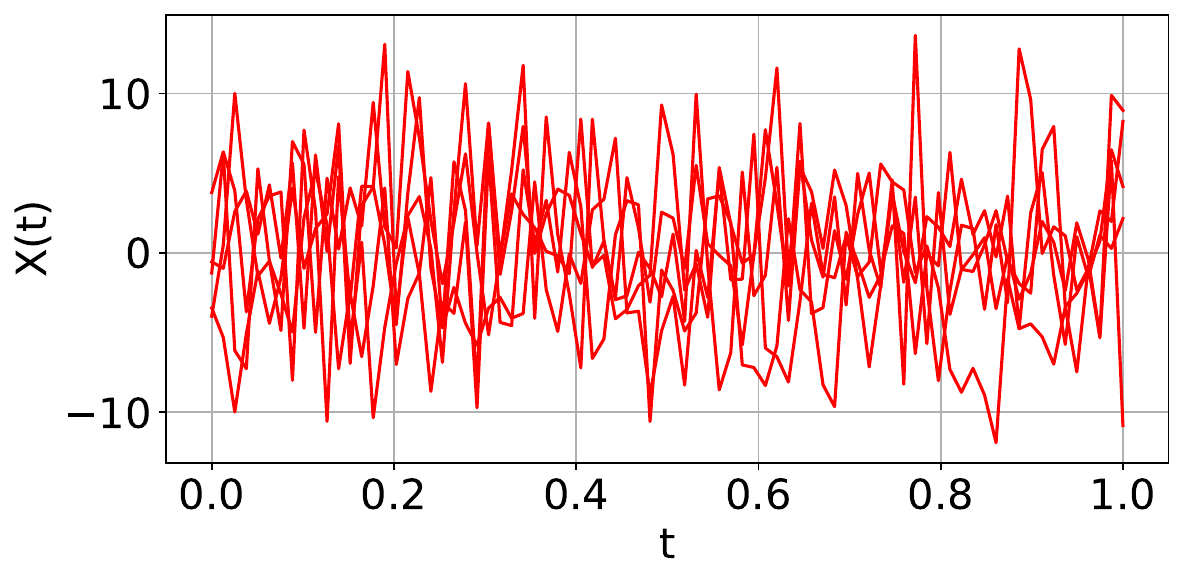} &
      \includegraphics[width=\imgw]{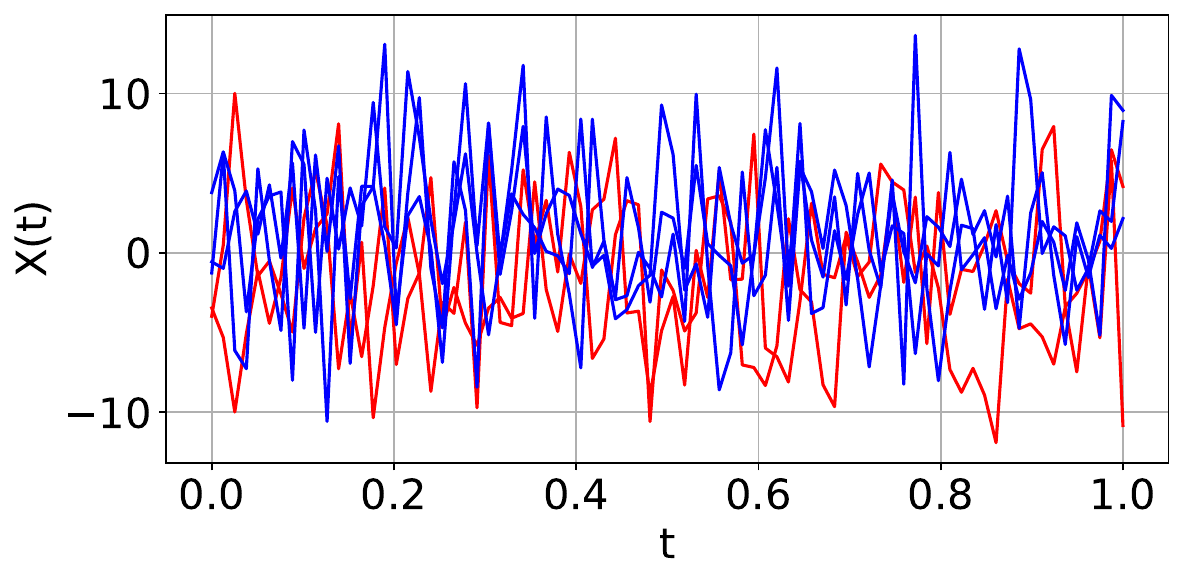} &
      \includegraphics[width=\imgw]{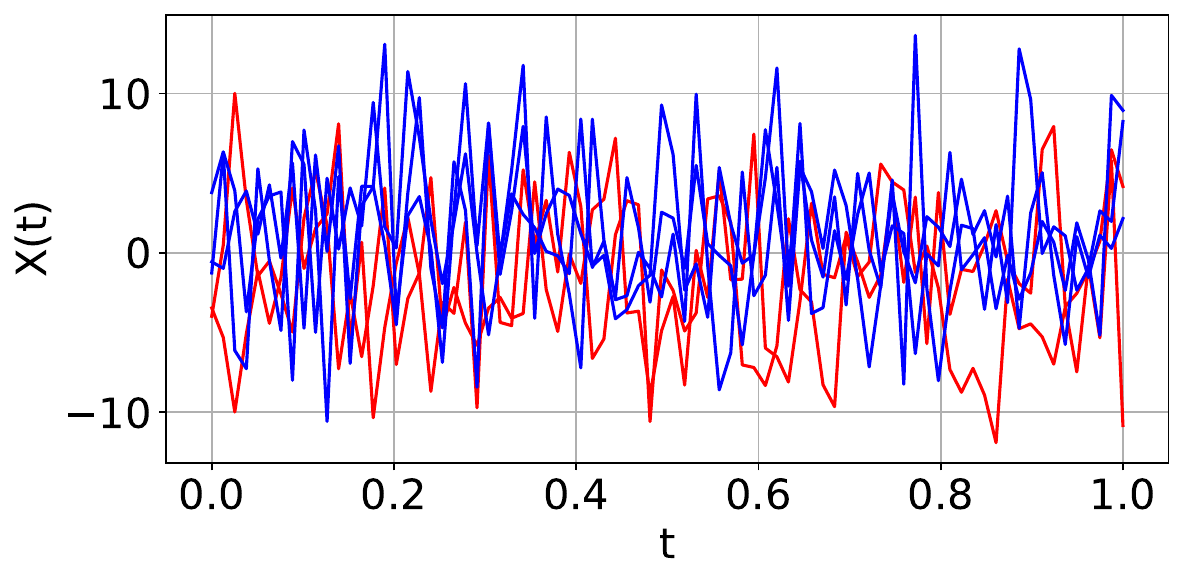} \\
    \textbf{$N=160$} &
      \includegraphics[width=\imgw]{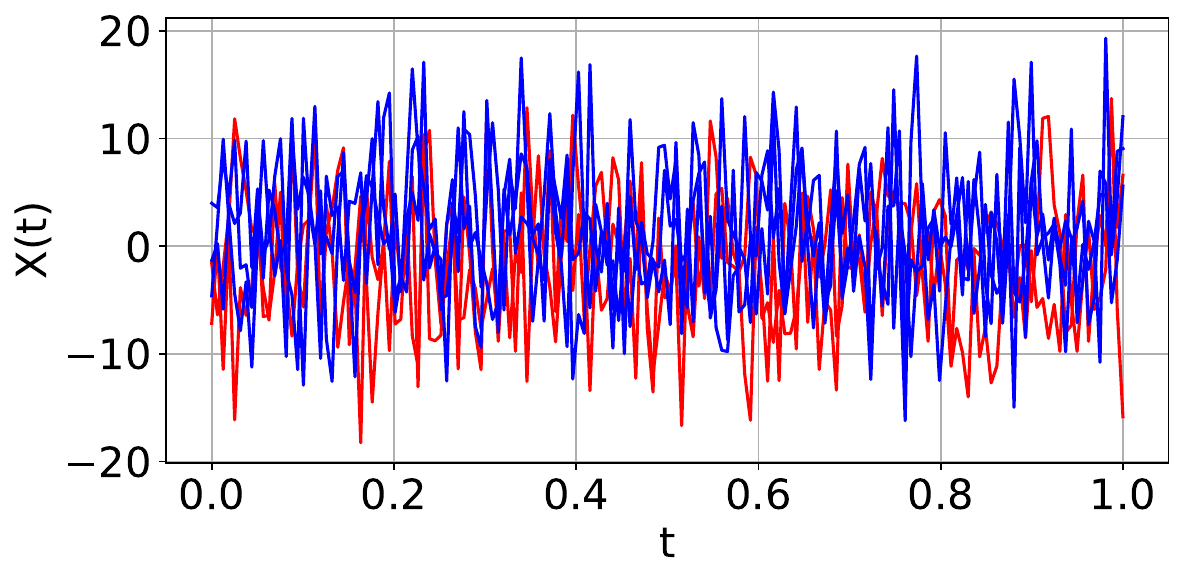} &
      \includegraphics[width=\imgw]{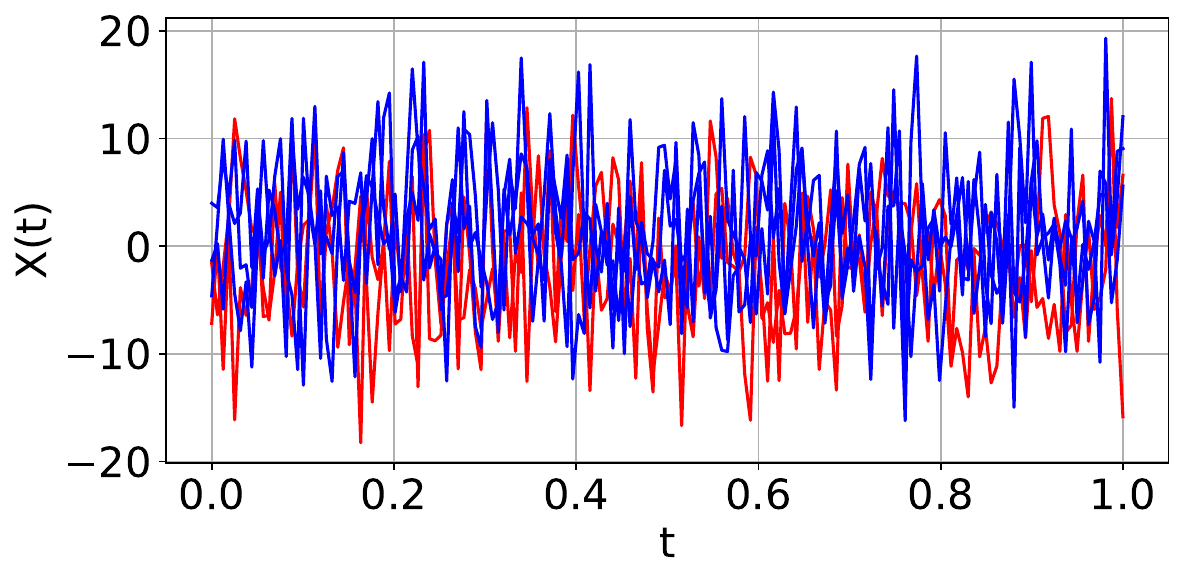} &
      \includegraphics[width=\imgw]{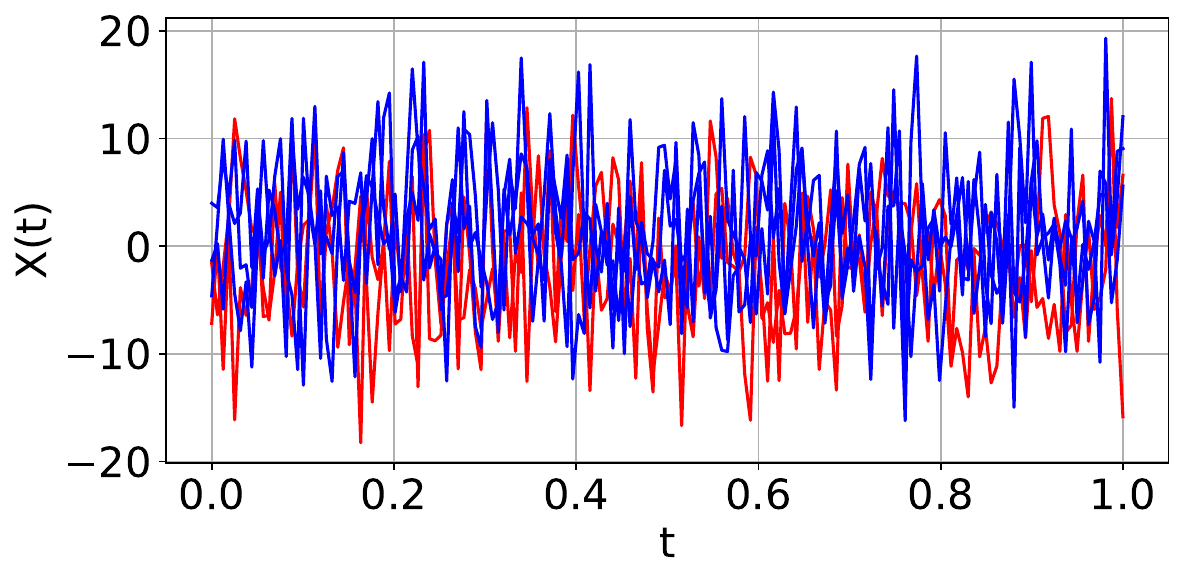} \\
    \textbf{$N=320$} &
      \includegraphics[width=\imgw]{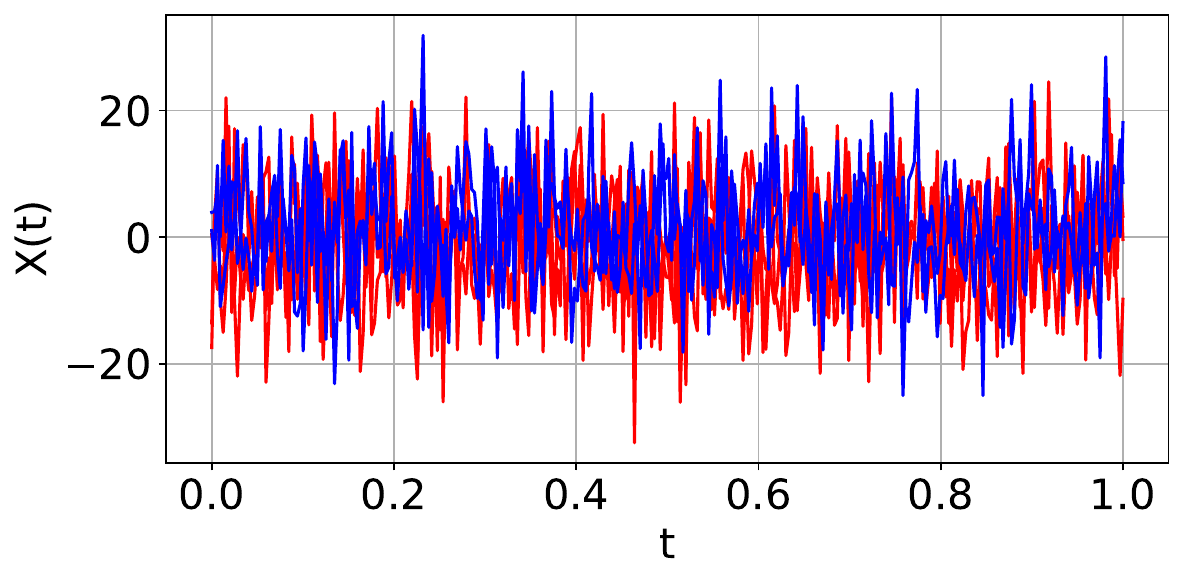} &
      \includegraphics[width=\imgw]{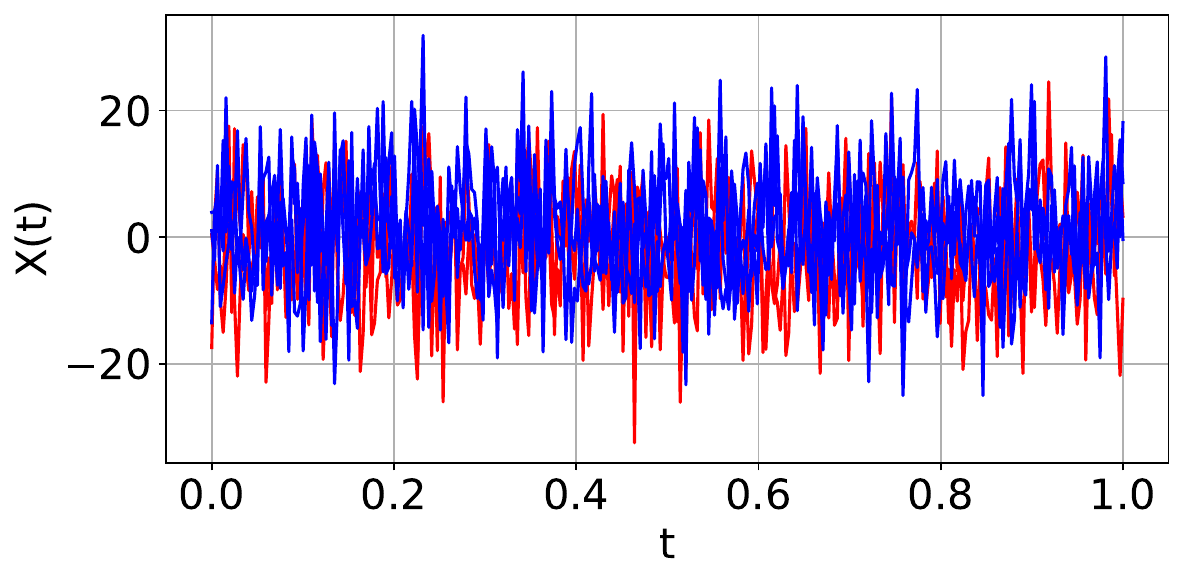} &
      \includegraphics[width=\imgw]{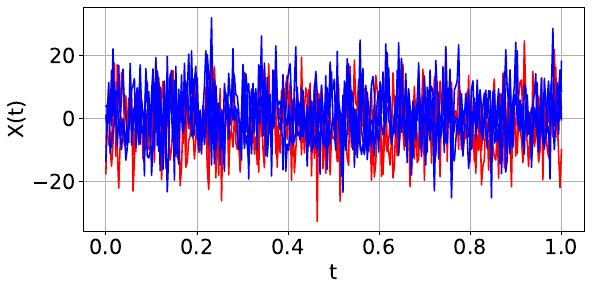} \\
    \textbf{$N=640$} &
      \includegraphics[width=\imgw]{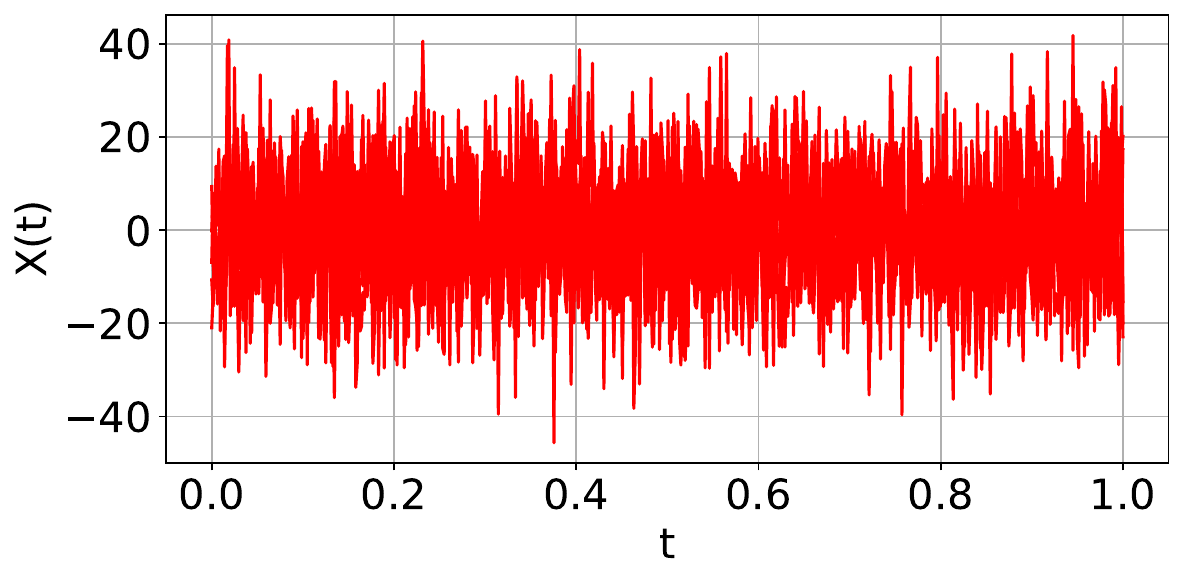} &
      \includegraphics[width=\imgw]{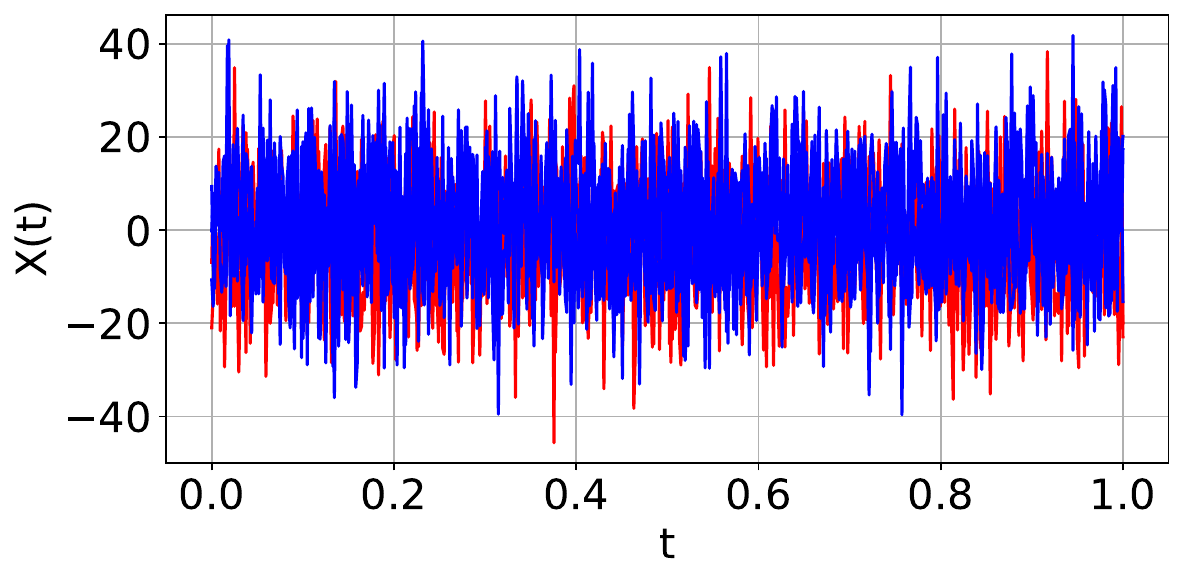} &
      \includegraphics[width=\imgw]{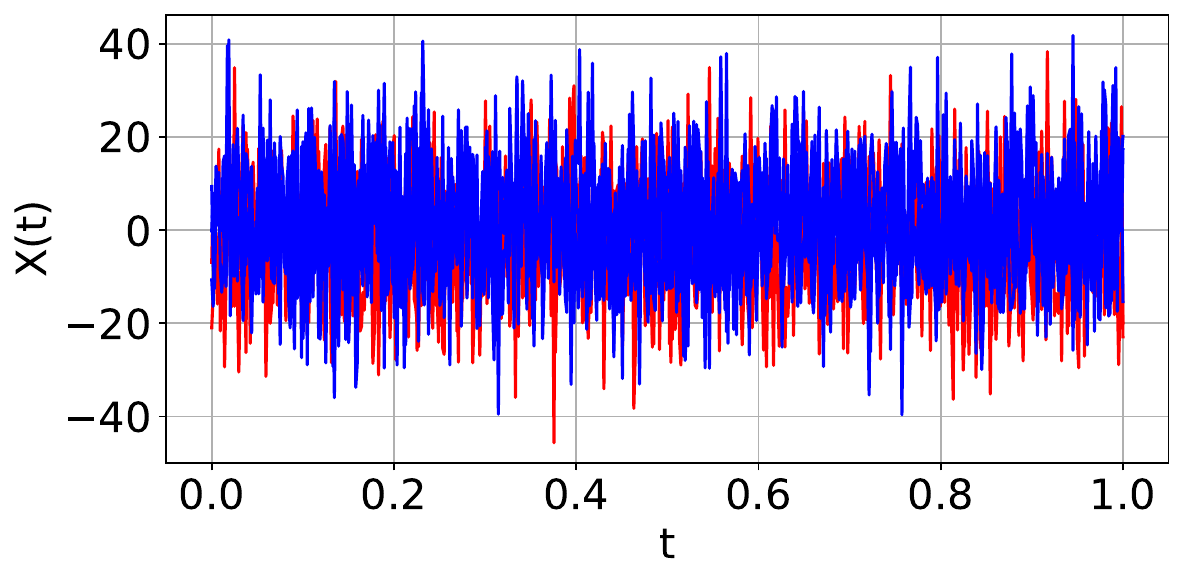} \\
    \multicolumn{1}{c}{} &
      \makebox[\imgw][c]{\textbf{Unnormalized Lap. Eucl.}} &
      \makebox[\imgw][c]{\textbf{Normalized Lap. $\Hk{-s}$}} &
      \makebox[\imgw][c]{\textbf{True label}} \\
  \end{tabular}

  \caption{Results for shifted white noise data for different resolutions (row labels) and different methods (column labels). See Section \ref{sec:white_noise} for details. Whilst at low resolutions both methods perform similarly, the labeling accuracy of the unnormalized Laplacian with the Euclidean geometry drops at higher resolutions while the normalized Laplacian with the $\Hk{-s}$ geometry ($s=1.01$) performs consistently well at all resolutions.}
  \label{fig:whitenoise_grid_panel}
\end{figure}

This is further illustrated in Figure \ref{fig:experiment_shift_with_witenoise_multi_realization_girds}, which shows the dependence of the labeling accuracy of different methods as a function of the resolution.
We see that the classification error of the unnormalized Laplacian with Euclidean geometry increases with resolution. The reported mean, maximum and minimum errors are computed by repeating the  experiment across 50 independently generated data sets. This degradation occurs because white noise almost surely does not belong to $\Lp{2}([0,1])$. For the same reason, the normalized Laplacian with the $\Lp{2}$ geometry performs equally poorly. By contrast, with the $\Hk{-s}$ geometry, our framework demonstrates resolution-independence and achieves accurate labeling even at high resolutions.

\begin{figure}[htp!]
    \centering
    
        \includegraphics[width=\textwidth]{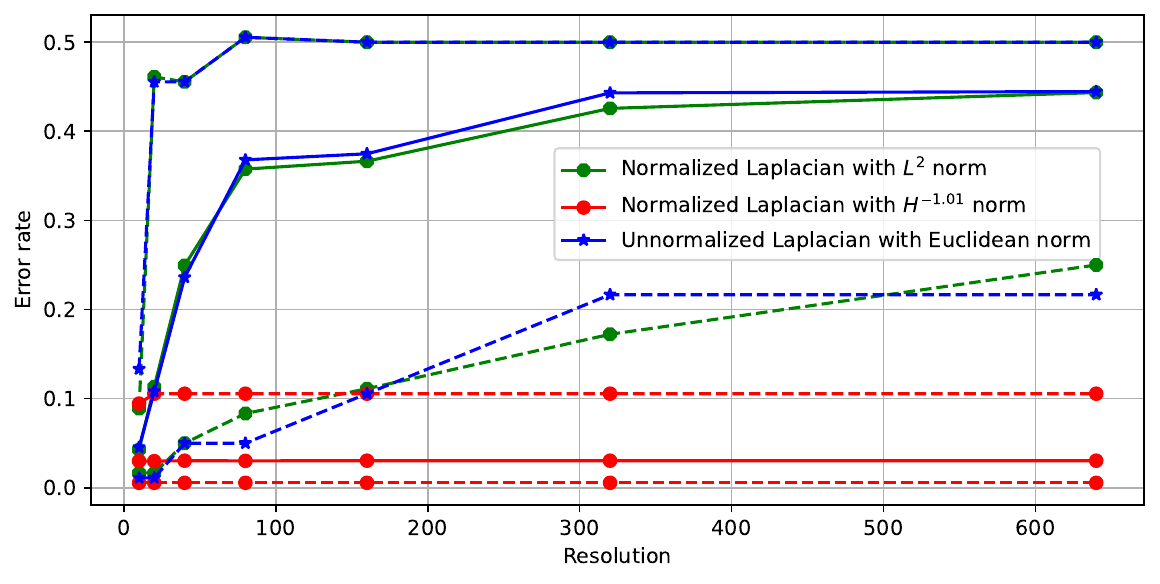}
        \caption{Labeling error for shifted white noise data at different resolutions. Solid line: mean error; dashed lines: maximum and minimum errors. See Section \ref{sec:white_noise} for details. The accuracy of the normalized Laplacian with $\Hk{-s}$ geometry ($s=1.01$) remains good across all resolutions whilst that of the unnormalised Laplacian with the Euclidean geometry and the normalized Laplacian with $\Lp{2}$ geometry deteriorates at high resolutions.}
        \label{fig:experiment_shift_with_witenoise_multi_realization_girds}
\end{figure}

\subsection{Brownian Bridge with Spatial Data}
\label{sec:Brownian_bridge}

\subsubsection{Features and Labels}

In this section, we consider infinite-dimensional feature vectors generated by a Brownian bridge between 0 and 1, i.e. $B_t:=\left(W_t \mid W_1=0\right)$, $t \in[0, 1]$ with $W_t$ being the standard Wiener process. Given the definition of Brownian motion, we have $B(t) \sim \mathcal{N}(0, t(1-t))$. By \cite{benyi2011modulation} we have that $B_t \in \Hk{s}([0,1])$ for any $s < 1/2$. For $X=\Lp{2}([0,1])$, the (unnormalized) eigenfunctions of the covariance are given by $\{\sin(k \pi \cdot)\}_{k = 1, 2, \ldots}$
and the eigenvalues by $\lambda_k = (k \pi)^{-2}$ (e.g.~\cite[Sec. 2.1]{foster2023brownian}); hence, they do not satisfy the assumptions of Theorem~\ref{thm:Setting:PointCons}. However, with a similar argument as in Section~\ref{sec:white_noise}, we get that the assumptions are satisfied if we take $X=\Hk{-s}([0,1])$ with any $s>0$.  

The system $\{\sqrt{2}\sin(k \pi \cdot)\}_{k = 1, 2, \ldots}$ is an orthonormal basis of $X=\Lp{2}([0,1])$, hence we choose
\begin{equation}\label{eq:ei-fi}
    \{e_k\}= \{\sqrt{2}\sin(k \pi \cdot)\}_{k = 1, 2, \ldots} \quad \text{and} \quad \{ f_k\}_{k=1,2,\ldots} = \{\left(1+(k \pi)^2\right)^{s / 2} e_k\}_{k=1,2,\ldots},
\end{equation}
so that $\{ f_k\}_{k=1,2,\ldots}$ forms an orthonormal basis of $\Hk{-s}([0,1])$, assuming  periodic boundary conditions (which are satisfied for the Brownian bridge).

We consider the following four settings for all labeling tasks:
\begin{enumerate}[topsep=3pt]
    \item unnormalized graph Laplacian for Euclidean (discretized) data;
    \item $X= \Hk{-s}([0,1])$ with $s=0.01$; because of the  small Sobolev exponent we don't expect a significant numerical difference from the $\Lp{2}$ setting (which, up to scaling, is equivalent to the Euclidean one)
    \item $X= \Hk{-s}([0,1])$ with $s=2$; in this setting we expect  low frequencies to be emphasized;
    \item $X= \Hk{-s}([0,1])$ with $s=0.01$ and a reweighted norm  as in~\eqref{eq:weighted-norm}, with a weighting given by 
    \begin{equation*}
    c_i= \begin{cases}50, & \text { if } i \in\{9,10,11\}, \\ 1, & \text { otherwise, }\end{cases}  
    \end{equation*}
  which assigns higher weight to  modes that we assume to be the most informative.
\end{enumerate}

We consider three labeling tasks for Brownian-bridge feature vectors evaluated on uniform finite grids. For each sample $j$, the feature vector is defined as
$$
\bar{x}_j:=\left[x_j\left(t_1\right), \ldots, x_j\left(t_N\right)\right]^{\top} \in \mathbb{R}^N,
$$
where $0=t_1<\cdots<t_N=1$ is a uniform partition of $[0,1]$. These vectors correspond to discretized sample paths of the Brownian bridge. To generate them, we first simulate a Brownian motion $W_t$ using Gaussian increments $\Delta W_i \sim \mathcal{N}(0, \Delta t)$ and then condition this process to return to zero at time $t=1$ using the standard transformation
$$
B_t=W_t-t W_1 .
$$

Given these feature vectors, we consider the following three labeling tasks.

\begin{enumerate}
    \item Average value labeling
    $$
y_j = \begin{cases} +1 & \text { if } \sum_{i=1}^{N} x_j(t_i) > 0 \\ 
-1 & \text { else.}\end{cases}
$$
    \item High frequency labeling
$$
y_j = \begin{cases} +1 & \text { if } \sum_{i=1}^{N} 
\sin(10\pi t_i) x_j(t_i)  > 0 \\ 
-1 & \text { else, }\end{cases}
$$
where the weighted sum $1/N \sum_{i=1}^N  \sin \left(10 \pi t_i\right) x_j\left(t_i\right)$ is an approximation of $\int_0^1 \sin (10 \pi t) x_j(t) \, \dd t$, i.e. the projection of the sample $x_j$ onto the 10th mode.

    \item Maximum value labeling
        $$
y_j = \begin{cases} +1 & \text { if } \max_{i\in \{1,\ldots, N \}} x_j(t_i) > u_b \\ 
-1 & \text { else. }\end{cases}
$$
We select $u_b = \sqrt{-\log(0.5)/2} $ such that the true labels split the dataset into two equally-sized classes, i.e. with  probability $50\%$ the maximum value of the feature vector will exceed $u_b$ \cite[Section 4.3]{karatzas2012brownian}: 
$$\mathbb{P}\left\{\max _{0 \leq t \leq T} B_t \geq \sqrt{-\log(0.5)/2} \mid B_T=0\right\}=0.5.$$
\end{enumerate}

\subsubsection{Results}
We fix $N=641$ grid points, given by $t_i=(i-1) / 640$.
For each data set, we generate 200 feature vectors and select $\ell=20$ of them as training labels.  The labels generated for each task are evenly balanced, with $50 \%$ in each class. 
We use $m=320$ modes in the basis expansion of each feature vector in accordance with the Nyquist–Shannon condition. 
The reported error statistics are computed by repeating the semi-supervised learning experiment on $50$ independently generated data sets. 

In what follows, we report the results obtained using the unnormalized Laplacian as in  Section~\ref{sec:laplace_learning_euclidean_spaces} compared to the normalized Laplacian with three different $\Hk{-s}$ norms 
as described above. 
For the infinite-dimensional  framework from Section \ref{sec:laplace_learning_func_spaces}, we reconstruct infinite-dimensional data on $X = \Lp{2}$ and $X = \Hk{-s}$   using the (truncated) bases $\{e_k\}_{k=0,...,m}$ and $\{f_k\}_{k=0,...,m}$, respectively, as defined in~\eqref{eq:ei-fi}. The coefficients are estimated via \eqref{eq:coef_learning_no_weight}. Hence each  data point takes the form $\sum_{k=1}^{m} \bar a_k \sin(k\pi t)$.
The choice of norms and weights reflects our prior knowledge about which modes dominate each task: high-frequency labels depend primarily on high-frequency modes (and we assume that we have a rough idea about the range of these relevant nodes, in our case around the 10$^{th}$ mode), whereas max-value and average-value labels are influenced more strongly by low-frequency modes.

In Fig.~\ref{fig:3x3_results}, the first column displays examples of true labels and representative misclassified samples, and the second column reports the corresponding error statistics. The labeling error plots are based on the results obtained using the norms that yield the lowest mean errors, that is the $\Hk{-2}$ norm for the average and maximum-value labels, and the weighted $\Hk{-0.01}$ norm for the high-frequency labels.

As expected, the normalized Laplacian using the $\Hk{-s}$ norm with $s = 0.01$ performs almost exactly as the unnormalized Laplacian in all experiments.
From Experiments~1 and~3, we observe that for  tasks dominated by low-frequency modes, the $\Hk{-2}$ norm, which places higher weight on low-frequency components, performs better than the unnormalized Laplacian with Euclidean norm and the $\Hk{-s}$ norm with $s = 0.01$. This indicates that the flexibility in selecting the norm allows the infinite-dimensional spectral method to outperform the Euclidean-space method.
Additionally, when prior knowledge about specific dominant higher-frequency modes is available, as demonstrated in Experiment~2, selectively amplifying these modes can  improve performance. With the current ratio 50:1 between ``important'' and ``unimportant'' modes we get a mean error decrease from ca. $50\%$ (essentially, random labels) to ca. $30\%$ compared to uniform weighting. With a more aggressive reweighting ratio of 500:1 we get a further improvement to ca. $15\%$ labeling error.

\begin{figure}[H]
\centering
\newcommand{\cellheight}{3.8cm}
\begin{minipage}[t]{0.35\textwidth}
    \centering

    \parbox[c][\cellheight][c]{\textwidth}{
        \includegraphics[width=\textwidth]{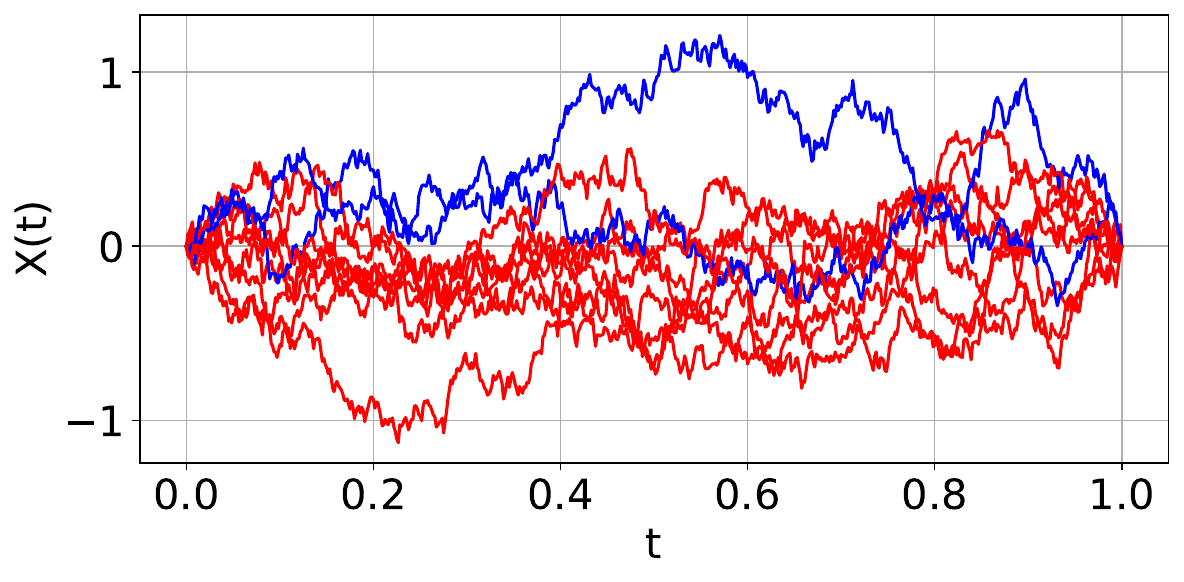}
    }
    \vspace{-2.em}
    \caption*{Avg value labeling: True labels}

    \parbox[c][\cellheight][c]{\textwidth}{
        \includegraphics[width=\textwidth]{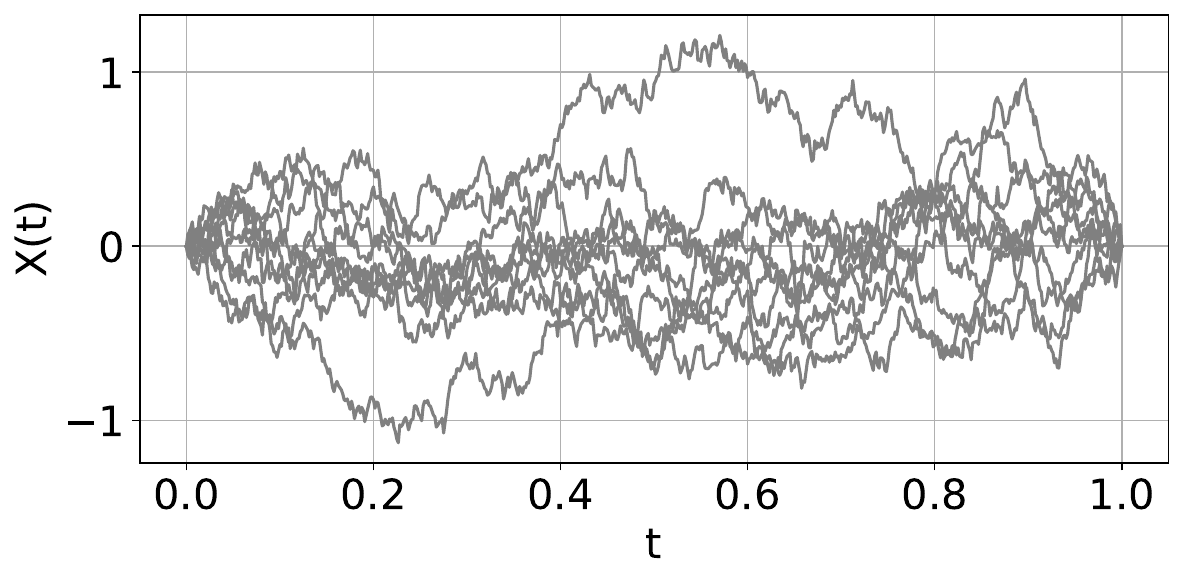}
    }
    \vspace{-2.em}
    \caption*{Avg value labeling: Predicted labels}
\end{minipage}
\hfill
\begin{minipage}[t]{0.6\textwidth}
    \centering    
    \parbox[c][\cellheight][c]{\textwidth}{
    \centering
    \begin{tabular}{lcccc}
    \hline
    Metric & Euc. & $s=0.01$ & $s = 2$ & $s=0.01$ weight.\\
    \hline
    Mean & 0.114 & 0.110 & 0.054 & 0.222\\
    Min & 0.011 & 0.011 & 0 & 0.044\\
    Max & 0.367& 0.361 & 0.178 & 0.494\\
    10\% & 0.039 & 0.038 & 0.022 & 0.122\\
    90\% & 0.238 & 0.237 & 0.106 & 0.361\\
    \hline
    \end{tabular}
    }
    \caption*{Average value labeling: Error statistics}
\end{minipage}

\begin{minipage}[t]{0.35\textwidth}
    \centering

    \parbox[c][\cellheight][c]{\textwidth}{
        \includegraphics[width=\textwidth]{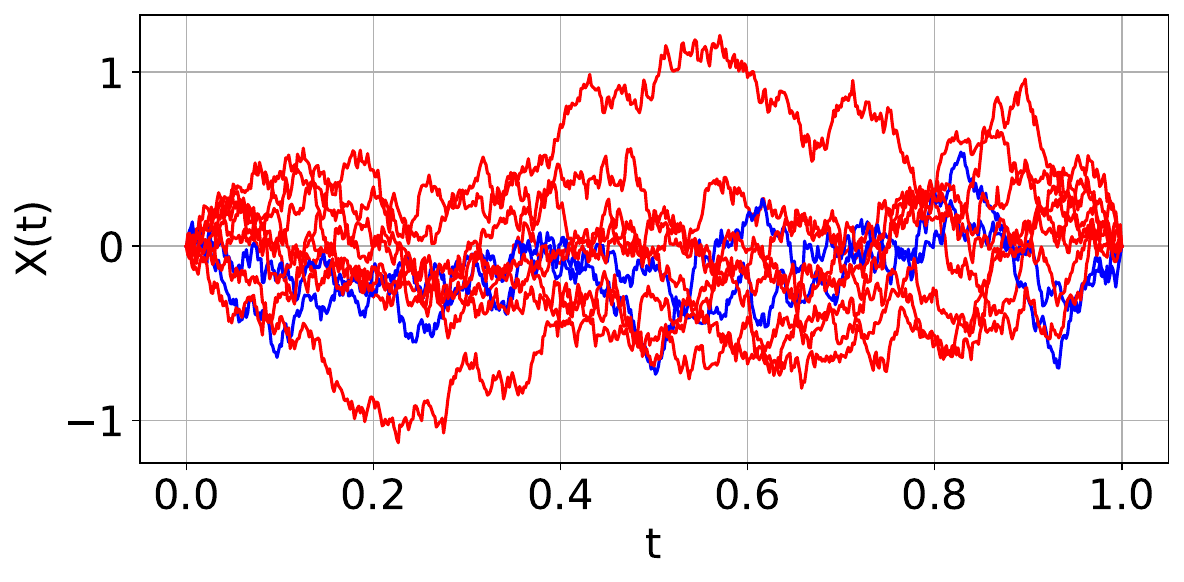}
    }
    \vspace{-2.em}
    \caption*{High freq  labeling: True labels}

    \parbox[c][\cellheight][c]{\textwidth}{
        \includegraphics[width=\textwidth]{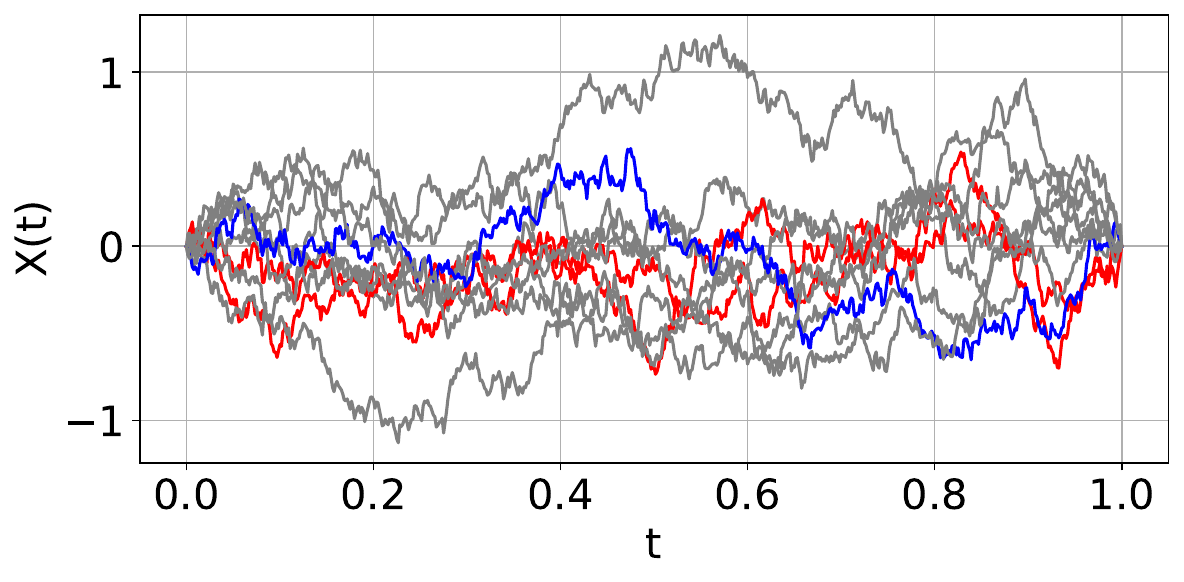}
    }
    \vspace{-2.em}
    \caption*{High freq  labeling: Predicted labels}
\end{minipage}
\hfill
\begin{minipage}[t]{0.6\textwidth}
    \centering    
    \parbox[c][\cellheight][c]{\textwidth}{
    \centering
    \begin{tabular}{lcccc}
    \hline
    Metric & Euc. & $s=0.01$ & $s = 2$ & $s=0.01$ weight.\\
    \hline
    Mean & 0.496 & 0.497 & 0.510 & 0.316\\
    Min & 0.417 & 0.411 & 0.444 & 0.1\\
    Max & 0.606 & 0.617 & 0.567 & 0.511\\
    10\% & 0.439 & 0.444 & 0.472 & 0.213\\
    90\% & 0.55 & 0.544 & 0.545 & 0.423\\
    \hline
    \end{tabular}
    }
    \caption*{High frequency  labeling: Error statistics}
\end{minipage}

\begin{minipage}[t]{0.35\textwidth}
    \centering

    \parbox[c][\cellheight][c]{\textwidth}{
        \includegraphics[width=\textwidth]{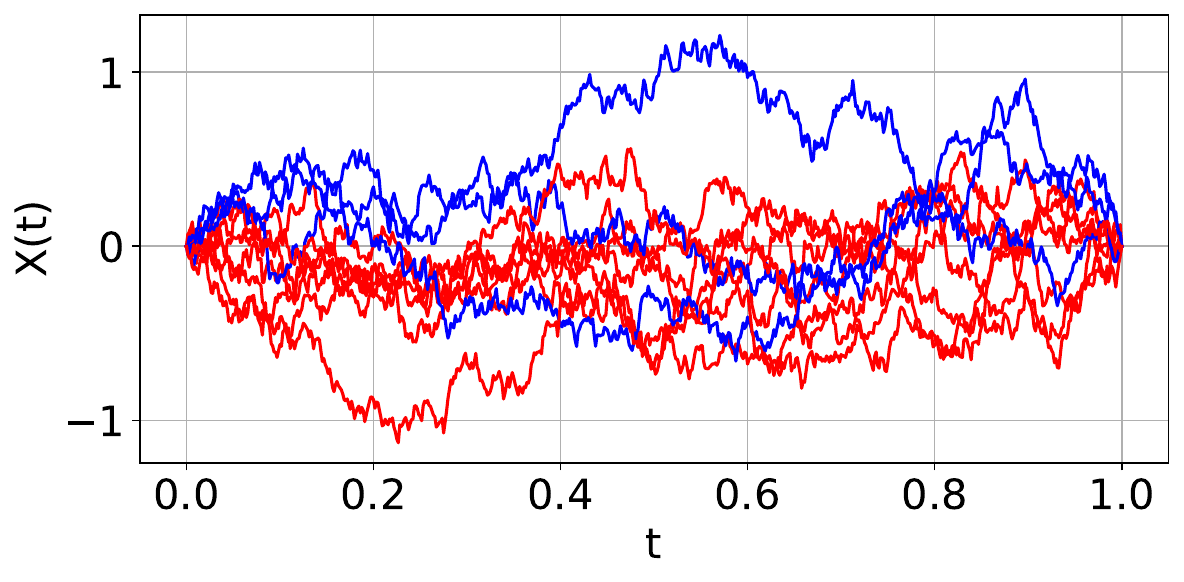}
    }
    \vspace{-2.em}
    \caption*{Max value labeling: True labels}

    \parbox[c][\cellheight][c]{\textwidth}{
        \includegraphics[width=\textwidth]{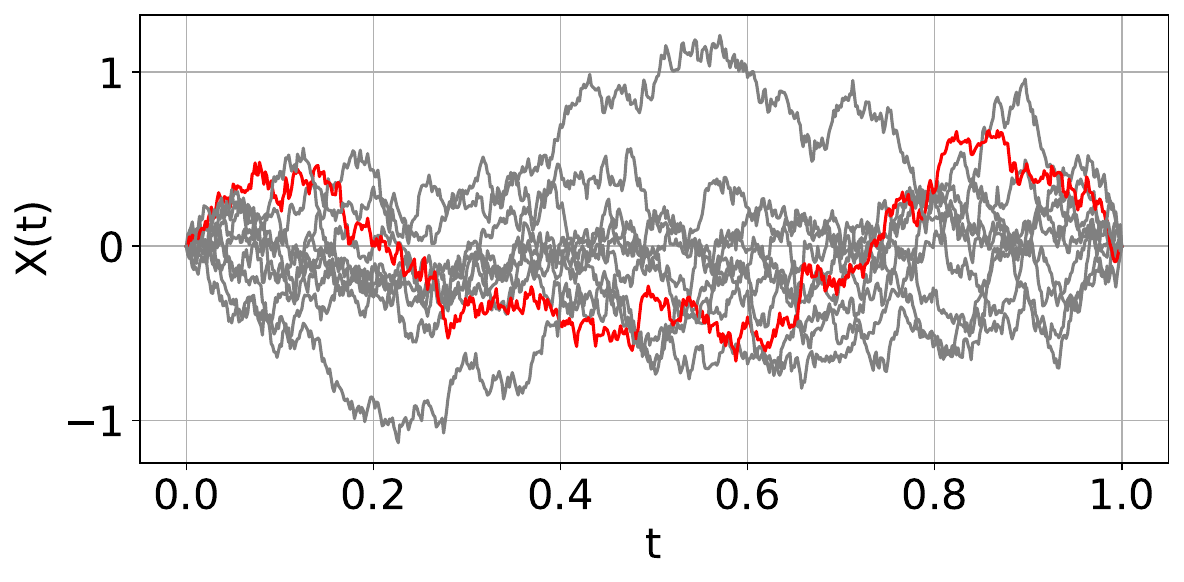}
    }
    \vspace{-2.em}
    \caption*{Max value labeling: Predicted labels}
\end{minipage}
\hfill
\begin{minipage}[t]{0.6\textwidth}
    \centering    
    \parbox[c][\cellheight][c]{\textwidth}{
    \centering
    \begin{tabular}{lcccc}
    \hline
    Metric & Euc. & $s=0.01$ & $s = 2$ & $s=0.01$ weight.\\
    \hline
    Mean & 0.213 & 0.210 & 0.155 & 0.310\\
    Min & 0.089 & 0.089 & 0.094 & 0.122\\
    Max & 0.467 & 0.467 & 0.261 & 0.467\\
    10\% & 0.128 & 0.122 & 0.111 & 0.192\\
    90\% & 0.339 & 0.334 & 0.184 & 0.445\\
    \hline
    \end{tabular}
    }
    \caption*{Maximum value labeling: Error statistics}
\end{minipage}

\vspace{0.6cm}

\vspace{-2.em}
\caption{Three experiments described in Section \ref{sec:Brownian_bridge}: true labels, error statistics computed using $50$ realizations, and visualization of the labeling error. The red and blue trajectories correspond to labels $-1$ and $+1$, respectively.}
\label{fig:3x3_results}
\end{figure}

\section{Conclusions} \label{sec:Conclusions}

In this paper we have established a pointwise limit for Laplace learning in infinite dimensions.
Our formulation is general in the sense that we give freedom in the choice of norm in the weights.
However, the pointwise limit does not imply convergence of the minimizers.
Before getting to that question a natural next step is to analyze the regularity of the continuum problem.
Regularity of elliptic PDEs in finite dimensions is well studied.
For example, in $d$ dimensions (for example $X\subset \bbR^d$ is open and bounded) we know that we need $p>d$ in order for the problem:
\[ \text{minimize } \int_X |\nabla u(x)|^p \, \dd \mu(x) \text{ s.t. } u(x_i) = y_i \text{ for all } i=1,\dots, N \]
to be well defined.
When $d=\infty$ it is not clear whether there is a well-posed problem for finite labels except when a Lipschitz assumption is imposed on $u$ and the above problem is reformulated accordingly (this is related to the so-called absolutely minimizing Lipschitz extensions which have been studied extensively in general metric spaces, e.g.~\cite{juutinen2002absolutely}).
It is also not clear how the choice of norm on $X$ will affect the regularity.
In finite dimensions all norms are equivalent, but this is not true in infinite dimensions which will lead to a norm-dependent regularity theory.

With a better handle of regularity one should be in a position to look for variational convergence.
This typically means a $\Gamma$-convergence result that would establish convergence of constrained minimizers to constrained minimizers of the continuum limit.
There are two steps in $\Gamma$-convergence: a liminf inequality and the existence of a recovery sequence.
If smooth functions over our Hilbert space $X$ are dense then the existence of a recovery sequence needs only be established over smooth functions, and therefore the result of this paper is likely to be sufficient.
The liminf inequality will require significant work to generalize from the finite-dimensional case.

Finally, the Gaussian assumption allowed us to perform many of the computations explicitly but it is of interest to consider the setting of a general data-generating measure with a density with respect to some Gaussian reference measure.

It's the goal of future work to address these open challenges.

\section*{Acknowledgments}

YK is grateful to the EPSRC (fellowship EP/V003615/2) which supported part of this research. 
MT is grateful for the support of the EPSRC Mathematical and Foundations of Artificial Intelligence Probabilistic AI Hub (grant agreement EP/Y007174/1) and the NHSBT award 177PATH25 ``Harnessing Computational Genomics to Optimise Blood Transfusion Safety and Efficacy''.
MT and ZZ are both supported by the Leverhulme Trust through the Project Award ``Robust Learning: Uncertainty Quantification, Sensitivity and Stability'' (grant agreement RPG-2024-051).
YK and MT would also like to thank the Isaac Newton Institute for Mathematical Sciences, Cambridge, for support and hospitality during the program Mathematics of deep learning (supported by EPSRC grant EP/R014604/1) where part of this work was discussed.

\bibliography{ref}{}
\bibliographystyle{plain}

\end{document}